\documentclass[journal]{elsarticle}

\usepackage{textcomp}
\usepackage[utf8]{inputenc}
\usepackage[T1]{fontenc}
\usepackage[english]{babel}
\usepackage{lmodern}
\usepackage{amsthm}
\usepackage{epsfig}
\usepackage{subfigure}
\usepackage{calc}
\usepackage{amssymb}
\usepackage{amstext}
\usepackage[shortlabels]{enumitem}
\usepackage{amsmath}
\usepackage{multicol}
\usepackage{siunitx}
\usepackage{pdflscape}

\usepackage{graphicx}
\usepackage{algpseudocode,algorithm,algorithmicx}
\usepackage{varwidth}
\usepackage{amsmath}
\usepackage{amssymb}
\usepackage{booktabs}
\usepackage{multirow}
\usepackage{tabularx}
\usepackage{array}
\usepackage{makecell}
\usepackage{comment}
\usepackage{wrapfig}
\usepackage{lineno}

\newtheorem{definition}{Definition}

\newtheorem{theorem}{Theorem}
\newtheorem{lemma}{Lemma}

\usepackage{textcomp}
\def\BibTeX{{\rm B\kern-.05em{\sc i\kern-.025em b}\kern-.08em
    T\kern-.1667em\lower.7ex\hbox{E}\kern-.125emX}}

\begin{document}
\begin{frontmatter}


\title{
Discovery of Crime Event Sequences with Constricted Spatio-Temporal Sequential Patterns
}

\author[address1]{Piotr S. Maciąg\corref{mycorrespondingauthor}}
\cortext[mycorrespondingauthor]{Corresponding author}
\ead{piotr.maciag@pw.edu.pl}

\author[address1]{Robert Bembenik}
\ead{r.bembenik@ii.pw.edu.pl}

\author[address2]{Artur Dubrawski}
\ead{awd@cs.cmu.edu}

\address[address1]{Warsaw University of Technology, Institute of Computer Science\\Nowowiejska 15/19, 00-665, Warsaw, Poland}
\address[address2]{Carnegie Mellon University, Auton Lab, The Robotics Institute\\ 5000 Forbes Avenue, Pittsburgh, PA 15213, USA}

\begin{abstract}

In this article, we introduce a novel type of spatio-temporal sequential patterns called Constricted Spatio-Temporal Sequential (CSTS) patterns and thoroughly analyze their properties. We demonstrate that the set of CSTS patterns is a concise representation of all spatio-temporal sequential patterns that can be discovered in a given dataset. To measure significance of the discovered CSTS patterns we adapt the participation index measure. We also provide \textit{CSTS-Miner}: an algorithm that discovers all participation index strong CSTS patterns in event data. We experimentally evaluate the proposed algorithms using two crime-related datasets: Pittsburgh Police Incident Blotter Dataset and Boston Crime Incident Reports Dataset. In the experiments, the CSTS-Miner algorithm is compared with the other four state-of-the-art algorithms: STS-Miner, CSTPM, STBFM and CST-SPMiner. As the results of experiments suggest, the proposed algorithm discovers much fewer patterns than the other selected algorithms. Finally, we provide the examples of interesting crime-related patterns discovered by the proposed CSTS-Miner algorithm.
\end{abstract}

\begin{keyword}
data mining, spatio-temporal sequential patterns, crime-data analysis, patterns discovery, concise representation
\end{keyword}

\end{frontmatter}


\section{Introduction}
\label{sec:Introduction}

Discovering knowledge in the form of various types of patterns, inference rules or motifs from spatio-temporal events data is a topic attracting increasing attention of researchers world-wide \citep{Han2011-DataMiningConcepts,Zaki2014-DataMiningAndAnalysis}. Specifically, many real-world spatio-temporal datasets consist of a set of event types and their event instances defined by geographical locations and time occurrence. An example of a spatio-temporal event dataset is a set of crime event incidents, such as arson, homicide or vandalism, each of which is assigned an event type, a geographical location and time of occurrence. Discovering sequences of crime types that occur in a spatial area over a time period can contribute to a better understanding of the causes of these crimes and to their elimination \citep{ref1284:Yu2016,He2021-Crime}. 

In order to discover such sequences of spatio-temporal event types, one can consider applying one of the algorithms for spatio-temporal sequential patterns discovery. A spatio-temporal sequential pattern (in brief, ST sequential pattern), introduced in \citep{ref1284:Huang2008}, is defined as a sequence of event types. By discovering ST sequential patterns, one can obtain insight into spatio-temporal relations between various event types. For example, the discovery of an ST sequential pattern \textit{arson $\rightarrow$ vandalism $\rightarrow$ bomb} can lead to the critical behavior pattern of a dangerous terrorist. 

In the literature, several types of methods and algorithms were already developed for the discovery of ST sequential patterns (see, for example, \citep{ref1284:Huang2008,ref1284:Mohan2012,ref1284:Maciag2019,ref1284:Maciag2019-Kes}). \citep{ref1284:Huang2008} introduced the first algorithm called STS-Miner for the discovery of significant ST sequential patterns. \citep{ref1284:Maciag2019,ref1284:Maciag2019-Kes} define a significant ST sequential pattern as a pattern, whose participation index measure is greater than the user-specified threshold $\theta$. \citep{ref1284:Maciag2019-Kes} refers to such a pattern as a PI-strong ST sequential pattern (we adapt this naming convention in this paper).  

While the already proposed algorithms (such as STS-Miner \citep{ref1284:Huang2008}, STBFM \citep{ref1284:Maciag2019}, CST-SPMiner \citep{ref1284:Maciag2019-Kes}, CSTPM \citep{ref1284:Mohan2012}, STES \citep{Aydin2016}) can discover PI-strong (closed) ST sequential patterns in some datasets, in practice, the number of discovered redundant patterns can still be too huge to be analyzed by a user of the algorithm. Hence, in this paper, we offer a novel representation of ST sequential patterns which we call \textit{Constricted Spatio-Temporal Sequential patterns} (in brief, CSTS patterns) and analyze their theoretical properties. As presented in the paper, given the set of CSTS patterns one can approximate participation index values of all ST sequential patterns. 

To verify the efficiency and effectiveness of the proposed approach, we use two real-world datasets of crime events: the Pittsburgh Police Incident Blotter Dataset and the Boston Crime Incident Reports Dataset. For example, one of the conducted experiments shows that for the Boston Crime Incidents Reports dataset, the proposed approach is able to discover \num{65967} CSTS patterns, while the three algorithms discovering all spatio-temporal sequential patterns provide as many as \num{228285} patterns. The discovered \num{65967} CSTS patterns can be used to derive all \num{228285} spatio-temporal sequential patterns and approximate participation index of each of them with a maximal error of $\pm~0.025$. 

\subsection{Contributions}

The contributions of the paper are as follows:
\begin{itemize}
    \item We introduce the notion of \textbf{a Constricted Spatio-Temporal Sequential (CSTS) pattern} that constitutes concise representation of all ST sequential patterns. 
    
    
    \item We thoroughly \textbf{analyze theoretical properties of CSTS patterns}. Specifically, we show that the set of CSTS patterns is a subset of the set of closed ST sequential patterns and that each CSTS pattern is also a closed ST sequential pattern. Moreover, we show that given the set of Participation Index (PI-)strong CSTS patterns one can obtain the set of all PI-strong ST sequential patterns and approximate participation index of each of them with an approximation margin of $\pm~\varepsilon$.
    
    \item We offer \textbf{a new algorithm called CSTS-Miner} that discovers PI-strong CSTS patterns.  CSTS-Miner applies \textbf{an introduced MAX-Tree structure} for more efficient candidate patterns generation. The proposed MAX-Tree is generated in two main phases of CSTS-Miner: the first phase called "top-down", in which all PI-strong ST sequential patterns are generated using the breadth-first strategy, and the second phase called "bottom-up", which calculates PI-strong CSTS patterns. We analyze properties and computation times of CSTS-Miner.
    
    \item \textbf{We experimentally compare the results obtained with the CSTS-Miner algorithm with three other state-of-the-art algorithms} discovering ST sequential patterns: the adapted version of STS-Miner \citep{ref1284:Huang2008}, STBFM \citep{ref1284:Maciag2019}, CSTPM \citep{ref1284:Mohan2012}, which discover PI-strong ST sequential patterns. Moreover, we also compare CSTS-Miner with the CST-SPMiner algorithm \citep{ref1284:Maciag2019-Kes}, which discovers PI-strong closed ST sequential patterns. Similarly to CSTS patterns, closed spatio-temporal sequential patterns discovered by CST-SPMiner also constitute a concise representation of all ST sequential patterns. For the purpose of experiments, we selected and preprocessed two crime events datasets: the Pittsburgh Police Incident Blotter Dataset and the Boston Crime Incident Reports Dataset. As we show, CSTS-Miner discovers much fewer redundant patterns than the other selected algorithms. Specifically, as the results of the experiments confirm, CSTS-Miner provides up to 60\% fewer patterns compared to the STS-Miner, STBFM and CSTPM algorithms and up to 40\% fewer patterns compared to the CST-SPMiner algorithm. 
    
    \item To the best of our knowledge, for the first time in the literature we provide experimental comparison of the effectiveness and efficiency of the above-mentioned STS-Miner, CSTPM, STBFM and CST-SPMiner algorithms discovering  (closed) spatio-temporal sequential patterns. Our implementations (prepared in the C++ language) of the selected algorithms (STS-Miner, STBFM, CSTPM, CST-SPMiner) as well as the proposed CSTS-Miner algorithm are available at the GitHub repository\footnote{https://github.com/piotrMaciag32/CSTS-Miner}.
    
    \item Eventually, \textbf{we provide examples of interesting crime-related patterns discovered by CSTS-Miner} from the Pittsburgh Police Incident Blotter Dataset.  
\end{itemize}

\subsection{Structure}

The structure of the article is as follows. In Section~\ref{sec:Related Work}, we offer a brief review of the related work. Section~\ref{subsec:Basic Notions} offers basic notions of ST sequential patterns. In Section~\ref{sec:ConstrictedMaximal}, we introduce the notion of a CSTS pattern. In Section~\ref{sec:TheoreticalProperties}, we analyze theoretical properties of CSTS patterns. Section~\ref{sec:CSTS-Miner} describes the proposed CSTS-Miner algorithm. In Section~\ref{sec:ExperimentalEvaluation}, we provide the results of experiments and in Section~\ref{sec:Conclusions} we conclude the article.

\section{Related Work}
\label{sec:Related Work}


The discovery of concise representations of various patterns (especially frequent patterns and sequential patterns) attracts researchers' attention. In \citep{ref1284:Yan2003}, \textit{closed sequential patterns} representation was introduced for the first time. Following \citep{ref1284:Yan2003}, numerous works were dedicated to the problem of more efficient mining of closed sequential patterns. The examples include methods and algorithms offered in \citep{ref1284:Wang2007Frequent,ref1284:Wang2004:BIDEClosedSeq} or, more recently, in \citep{ref1284:Fumarola2016} and \citep{ref1284:Gomariz2013}. Other related research directions include discovery of top closed sequential patterns with the highest support \citep{ref1284:Tzvetkov2003,Zhang2015} and parallel discovery of closed sequential patterns \citep{ref1284:Cong2005}. A survey of the current methods for the discovery of closed sequential patterns can be found in \citep{ref1284:Fournier-Viger2017}. 

While many methods and algorithms were offered to discover various types of spatio-temporal patterns, relatively few of them focused on discovering ST sequential patterns. To this end, in our experiments the proposed CSTS-Miner algorithm is compared with the most representative algorithms mining ST sequential patterns, namely:
\begin{itemize}
    \item \textbf{STS-Miner} \citep{ref1284:Huang2008} - the first algorithm offered for the discovery of ST sequential patterns. STS-Miner uses the depth-first strategy. In this work, we adapted STS-Miner to use the participation index measure instead sequence index measure. Thus, the adapted version of STS-Miner is able to discover PI-strong ST sequential patterns. 

    \item \textbf{STBFM} \citep{ref1284:Maciag2019} - the algorithm discovering PI-strong ST sequential patterns by means of the breadth-first pattern generation strategy. \citep{ref1284:Maciag2019} introduced a structure called SP-Tree that allows to efficiently generate candidate patterns using their first and second parent patterns and a children list of the first parent pattern. In addition, \citep{ref1284:Maciag2019} presented the two variations of STBFM that can discover top-K PI-strong ST sequential patterns. The experiments of \citep{ref1284:Maciag2019} showed that STBFM is able to discover some interesting crime-related patterns. 
    
    \item \textbf{CST-SPMiner} \citep{ref1284:Maciag2019-Kes} - the algorithm discovering closed PI-strong ST sequential patterns. CST-SPMiner applies the breadth-first candidate patterns generation strategy to obtain all PI-strong closed ST sequential patterns. For each obtained PI-strong ST sequential pattern $\overrightarrow{s}$, CST-SPMiner determines if this pattern is closed or not using a check condition verifying if $\overrightarrow{s}$ is a closure pattern of both of its parent patterns. 
    
    \item \textbf{CSTPM} \citep{ref1284:Mohan2012} - the miner discovering Cascading Spatio-Temporal Patterns (CSTP). CSTP patterns consist of not only ST sequential patterns but also of cascades of event types. In our work, to directly compare CSTPM with the proposed CSTS-Miner, we adapted the CSTPM algorithm for the discovery of ST sequential patterns rather than cascading ST patterns. 
\end{itemize}

The reviews of methods and algorithms for (spatio-temporal) patterns discovery with particular emphasis on spatio-temporal event datasets can be found in \citep{ref1284:Kisilevich2010,ref1284:Li2014,ref1284:Sunitha2014, Aydin2016,ref1284:Maciag2017,ref1284:Maciag2018,ref1284:Atluri2018,Ansari2019-STClusteringReview,Aydin2020-Geoinformatica}.

\section{Basic Notions}
\label{subsec:Basic Notions}

The definitions presented in this section are formulated based on the works \citep{ref1284:Huang2008,ref1284:Maciag2019, ref1284:Maciag2019-Kes}.


Let $ \mathbf{F}$ denote a set of $n$ event types and $ \mathbf{D} $ denote a dataset of event instances in which each $ e \in \mathbf{D} $ consists of a spatial location, an occurrence time (timestamp) and an event type $ F \in  \mathbf{F} $. $ \mathbf{D}$ will be called a \textit{spatio-temporal event dataset}. Moreover, let $|\mathbf{D}|$ denote the number of event instances in $\mathbf{D}$. The set of all event instances of type $ F $ in dataset $ \mathbf{D} $ will be denoted by $ \mathbf{D}(F) $. 


Let us consider an example of a spatio-temporal event dataset presented in Figure~\ref{Fig:Dataset}. This dataset consists of event instances $\mathbf{D} = \{a_1, a_2, b_1,$ $ \dots, b_8, c_1, $ $\dots, c_8, d_1, \dots, d_3, $ $e_1, \dots, e_5\}$ and a set of five event types $\mathbf{F} = \{A, B, C, D, E\}$.\footnote{Please note that the spatial location of each event instance presented in Figure~\ref{Fig:Dataset} is specified for simplicity using only one dimension. However, in real datasets, the spatial location of event instances is usually defined by coordinates of two dimensions (for example, in the datasets selected for the experiments, spatial location is defined using longitude and latitude coordinates).}

\begin{figure}[h!t]
	\centering 
	\includegraphics[width=1.0\linewidth]{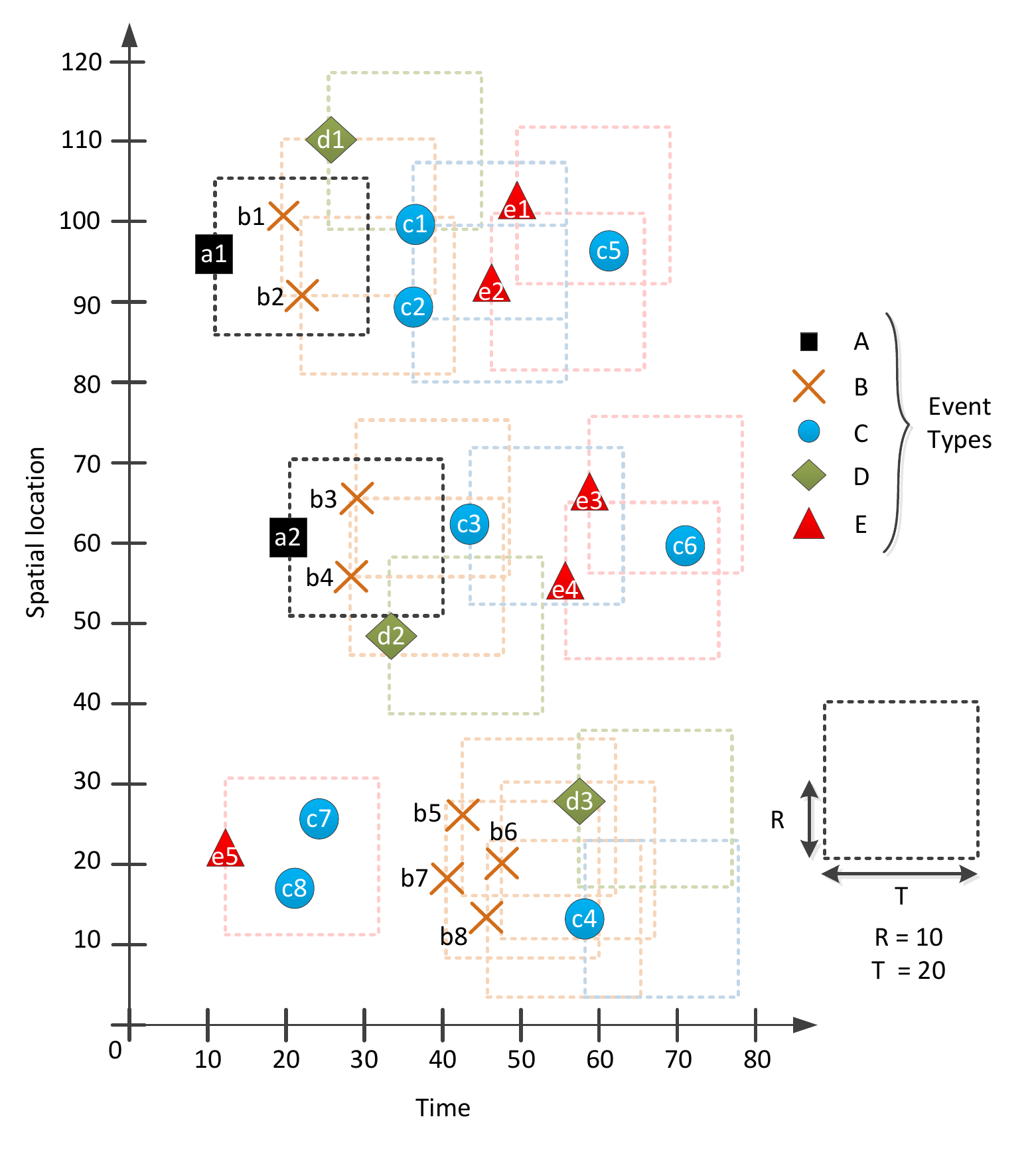}
	\caption{The example of a spatio-temporal event dataset. The figure also presents neighborhood specification of event instances.}
	\label{Fig:Dataset}      
\end{figure}

Spatio-temporal sequential pattern (ST sequential pattern) $\overrightarrow{s}$ is defined as a sequence of elements, each of which is an event type from $\mathbf{F}$ \citep{ref1284:Huang2008}. Please note that an event type $F \in \mathbf{F}$ can occur multiple times in an ST sequential pattern $\overrightarrow{s}$. Before we provide a formal definition of an ST sequential pattern, we recall the definition of a spatio-temporal neighborhood of an event instance with respect to an event type:

\begin{definition}[Neighborhood of event instance with respect to event type \citep{ref1284:Huang2008,ref1284:Maciag2019-Kes}]
	For an event instance $ e $, the \textit{neighborhood of} $ e $ \textit{with respect to an event type} $ F \in \mathbf{F}$ is denoted by $ \mathbf{N}(e, F) $ and is defined as follows:
	\begin{align}
	     \mathbf{N}(e, F) = & \{p~|~p \in \mathbf{D}(F) \nonumber \\ 
	    & \land distance(p.location, e.location) \leq R \\
	    & \land (p.time - e.time) \in (0, T] \}, \nonumber
	\end{align}
	where $R$ and $T$ are user-given spatial distance and time window thresholds, respectively. 
	\label{Def:Neighborhood}
\end{definition}

In Figure~\ref{Fig:Dataset}, the neighborhood $\mathbf{N}(a_1, B)$ consists of event instances $\{b_1, b_2\}$. Similarly, the neighborhood $\mathbf{N}(c_1, E) = \{e_1, e_2\}$. The spatial distance threshold of neighborhoods given in Figure~\ref{Fig:Dataset} is $R = 10$, while the time window threshold is equal to $ T = 20$. In our experiments, the spatial distance between locations of two event instances is calculated as the Euclidean distance between these locations. Please note that according to Definition~\ref{Def:Neighborhood}, an event instance $e_j$ can be located in a neighborhood of an event instance $e_i$, only if the occurrence time of $e_j$ is greater than the occurrence time of $e_i$ (i.e., the two event instances with the same occurrence time can not mutually belong to their neighborhoods because the difference between their occurrence time would be $e_i.time - e_j.time = 0$). 

\begin{definition}[Spatio-temporal sequential pattern]
	A \textit{spatio-temporal sequential pattern} (in brief, \textit{ST sequential pattern}) is a sequence of event types in $ \mathbf{F} $. The number of elements of the sequence is called its \textit{length}. \textit{i}-th element of sequence $ \overrightarrow{s} $ is denoted by $  \overrightarrow{s}[i] $. Sequence $ \overrightarrow{s} $, which consists of $ m $ elements is denoted as $ \overrightarrow{s}[1] \rightarrow \overrightarrow{s}[2] \rightarrow \dots \rightarrow \overrightarrow{s}[m]$. The number of elements of sequence $\overrightarrow{s}$ is defined as its length.
	\label{Def:STPattern}
\end{definition}

The example of an ST sequential pattern for the dataset presented in Figure~\ref{Fig:Dataset} is $\overrightarrow{s} = A \rightarrow B \rightarrow C$, the length of which is $3$. The important question is how to efficiently calculate neighborhoods of event instances? In our implementation, we adapted the plane sweep algorithm (see e.g. \citep{ref1284:Arge1998}) that has computational complexity of $\mathcal{O}(|\mathbf{D}(F)|\cdot log(|\mathbf{D}(F)|))$. 

\begin{definition} [Set of event instances supporting an element of an ST sequential pattern \citep{ref1284:Huang2008,ref1284:Maciag2019-Kes}]
	\textit{A Set of event instances supporting} \textit{i-th element of ST sequential pattern} $ \overrightarrow{s} $ is denoted by $\mathbf{I}(\overrightarrow{s},i)$ and is defined as follows:
	\begin{equation}
	\mathbf{I}(\overrightarrow{s}, i) = \left \{
		\begin{array}{ll}
		\mathbf{D}(\overrightarrow{s}[1]) & \text{ when }  i = 1 ,  \\
		\bigcup\limits_{e \in \mathbf{I}(\overrightarrow{s},i-1)} \mathbf{N}(e, \overrightarrow{s}[i]) & \text{ when }  i > 1.
		\end{array}
		\right.
		\end{equation}
	\label{Def:ISeq}
\end{definition}

For each ST sequential pattern $\overrightarrow{s}$, we can unambiguously distinguish sets of event instances supporting elements of that pattern. For the first element of a pattern, the set of event instances $\mathbf{I}(\overrightarrow{s}, 1)$ supporting that element is defined simply as all event instances of event type $\overrightarrow{s}[1]$ in $\mathbf{D}$. For every next element of $\overrightarrow{s}$ (say $i$), the set of supporting event instances $\mathbf{I}(\overrightarrow{s}, i)$ consists of all those event instances of event type $\overrightarrow{s}[i]$, which belong to neighborhoods of instances contained in the supporting set $\mathbf{I}(\overrightarrow{s}, i - 1)$.

Let us consider an example of a previously given ST sequential pattern $\overrightarrow{s} = A \rightarrow B \rightarrow C$ of the dataset in Figure~\ref{Fig:Dataset}. The sets of event instances supporting $\overrightarrow{s}$ are as follows:
\begin{itemize}
    \item $\mathbf{I}(\overrightarrow{s}, 1) = \mathbf{D}(A) = \{a_1, a_2\}$;
    \item $\mathbf{I}(\overrightarrow{s}, 2) = \bigcup\limits_{a \in \mathbf{I}(\overrightarrow{s},1)} \mathbf{N}(a, B) = \{b_1, b_2, b_3, b_4\}$;
    \item $\mathbf{I}(\overrightarrow{s}, 3) = \bigcup\limits_{b \in \mathbf{I}(\overrightarrow{s}, 2)} \mathbf{N}(b, C) = \{c_1,  c_2, c_3 \}$;
\end{itemize}

In this article, we apply the previously introduced in \citep{ref1284:Mohan2012,Aydin2016,ref1284:Maciag2019} \textit{participation ratio} and \textit{participation index} measures of significance of discovered patterns. The participation ratio of an $i$-th element of an ST sequential pattern $\overrightarrow{s}$ expresses the quotient of the number of event instances supporting $i$-th element of $\overrightarrow{s}$ to the number of event instances of $\overrightarrow{s}[i]$ event type in $\mathbf{D}$. 

\begin{definition}[Participation ratio \citep{ref1284:Maciag2019-Kes}]
	The participation ratio of an $i$-th element of ST sequential pattern $ \overrightarrow{s}$, where $i \geq 1$, is denoted by $PR(\overrightarrow{s}, i)$ and is defined as the ratio of the cardinality of the set of event instances supporting $i$-th element of $ \overrightarrow{s} $ to the number of all instances of type $ \overrightarrow{s}[i] $ in the dataset $\mathbf{D}$; that is: $PR(\overrightarrow{s}, i) = \dfrac{\big|I(\overrightarrow{s}, i)\big|}{\big|\mathbf{D}(\overrightarrow{s}[i])\big|}$.
	\label{Def:PR}
\end{definition}

By Definition~\ref{Def:PR}, the value of participation ratio is always in the range [0,1]. Participation index is defined as the minimal of participation ratios of all elements of an ST sequential pattern. 

\begin{definition}[Participation index \citep{ref1284:Maciag2019-Kes}]
	The \textit{participation index of ST sequential pattern} $ \overrightarrow{s} = \overrightarrow{s}[1] \rightarrow \overrightarrow{s}[2] \rightarrow \dots \rightarrow \overrightarrow{s}[m] $ is denoted by $PI(\overrightarrow{s})$ and is defined as the minimum from the participation ratios of all elements of $ \overrightarrow{s} $; that is, 
	$
	    PI(\overrightarrow{s}) = 
		\text{min}\big(\{PR(\overrightarrow{s}, i)| \ i = 1, 2, \dots, m\}\big) 
	$.
	\label{Def:PI}
\end{definition}

In Table~\ref{Table:PatternsSet}, we listed all ST sequential patterns for which participation indexes are greater than 0 for the dataset of Figure~\ref{Fig:Dataset}.

\begin{table}[h!t]
\caption{All ST sequential patterns and their participation indexes (provided in parentheses next to the patterns) for the dataset of Figure~\ref{Fig:Dataset}.}
\small{
\resizebox{\columnwidth}{!}{\begin{tabularx}{\linewidth}{ll}
	\toprule
    \textbf{Pattern length} & \textbf{Patterns set} \\
    \midrule
    $L_1 (F)$ & $A(1), B(1), C(1), D(1), E(1)$ \\
    \midrule
    $L_2$ & \makecell[l]{$ A \rightarrow B(0.5)$, $ B \rightarrow B(0.625) $, $ B \rightarrow C(0.5) $, \\$ B \rightarrow D(1) $, $ C \rightarrow E(0.8) $, $ E \rightarrow C(0.5) $} \\
    \midrule
    $L_3$ & \makecell[l]{$ A \rightarrow B \rightarrow B(0.25)$, $ A \rightarrow B \rightarrow C(0.375)$, \\$ A \rightarrow B \rightarrow D(0.5)$, $ B \rightarrow B \rightarrow C(0.5)$,\\
                        $ B \rightarrow B \rightarrow D(0.33)$, $ B \rightarrow C \rightarrow E(0.375)$,\\ $ C \rightarrow E \rightarrow C(0.25)$} \\
    \midrule
    $L_4$ & \makecell[l]{$ A \rightarrow B \rightarrow B \rightarrow C(0.25)$, \\$ A \rightarrow B \rightarrow C \rightarrow E(0.375)$, \\ $ B \rightarrow B \rightarrow C \rightarrow E(0.375)$,\\
                        $ B \rightarrow C \rightarrow E \rightarrow C(0.25)$} \\
    \midrule
    $L_5$ & \makecell[l]{$ A \rightarrow B \rightarrow B \rightarrow C \rightarrow E(0.25)$, \\$ A \rightarrow B \rightarrow C \rightarrow E \rightarrow C(0.25)$, \\ $ B \rightarrow B \rightarrow C \rightarrow E \rightarrow C(0.25)$} \\
    \midrule
    $L_6$ &  \makecell[l]{$ A \rightarrow B \rightarrow B \rightarrow C \rightarrow E \rightarrow C(0.25)$}\\
	\bottomrule
\end{tabularx}}
}
\label{Table:PatternsSet}
\end{table}

\begin{definition}[PI-strong ST sequential pattern]
    A candidate ST sequential pattern $\overrightarrow{s}$ is called PI-strong if its participation index $PI(\overrightarrow{s})$ is greater than the participation index threshold $\theta$.
    \label{Def:PIstrong}
\end{definition}

Discovery of all PI-strong ST sequential patterns can be performed, for example, using the STBFM algorithm introduced in \citep{ref1284:Maciag2019}. 

\begin{definition}[(Proper) Subsequence and (proper) supersequence of an ST sequential pattern] 
    Let $ \overrightarrow{s_1} = \overrightarrow{s_1}[1] \rightarrow \overrightarrow{s_1}[2] \rightarrow \dots \rightarrow \overrightarrow{s_1}[m_1] $ and $ \overrightarrow{s_2} = \overrightarrow{s_2}[1] \rightarrow \overrightarrow{s_2}[2] \rightarrow \dots \rightarrow \overrightarrow{s_2}[m_2] $ be ST sequential patterns. $ \overrightarrow{s_1}$ is a subsequence of $\overrightarrow{s_2} $ and $ \overrightarrow{s_2} $ is a supersequence of $ \overrightarrow{s_1} $ if  $ m_1 \leq m_2 $ and there exits an integer $k$, where $0 \leq k \leq m_2 - m_1$, such that $ \overrightarrow{s_1}[1] = \overrightarrow{s_2}[1+k] \land \overrightarrow{s_1}[2] = \overrightarrow{s_2}[2+k] \land \dots \land \overrightarrow{s_1}[m_1] = \overrightarrow{s_2}[m_1 + k]$. 
    
    If $ m_1 < m_2 $, then $ \overrightarrow{s_1} $ is a proper subsequence of $ \overrightarrow{s_2} $ and $ \overrightarrow{s_2} $ is a proper supersequence of $ \overrightarrow{s_1} $.  
	\label{Def:SubAndSup}
\end{definition}

\begin{theorem} [Anti-monotonicity property of the participation index for supersequences \citep{ref1284:Maciag2019-Kes}]
     Let $\overrightarrow{s_1}$ and $\overrightarrow{s_2}$ be ST sequential patterns. If $\overrightarrow{s_1}$ is a subsequence of $\overrightarrow{s_2}$, then $PI(\overrightarrow{s_1}) \geq PI(\overrightarrow{s_2})$ \footnote{We refer the reader to \citep{ref1284:Maciag2019-Kes} for the proof of the theorem.}.
     \label{theorem:1}
\end{theorem}

For the dataset presented in Figure~\ref{Fig:Dataset}, $\overrightarrow{s_1} = A \rightarrow B$ is a proper subsequence of $\overrightarrow{s_2} = A \rightarrow B \rightarrow C \rightarrow E$ ($\overrightarrow{s_2}$ is a proper supersequence of $\overrightarrow{s_1}$). As follows from Theorem~\ref{theorem:1}, the $PI$ value of an ST sequential pattern $\overrightarrow{s_2}$ is always less than or equal to the $PI$ value of its proper subsequence $\overrightarrow{s_1}$. The STBFM, CST-SPMiner and CSTPM algorithms apply Theorem~\ref{theorem:1} to efficiently generate candidate ST sequential patterns using the breadth-first search strategy.

\citep{ref1284:Maciag2019-Kes} introduced a concise and lossless representation of ST sequential patterns called closed ST sequential patterns. The important property of closed ST sequential patterns is that one can obtain the value of participation index of any ST sequential pattern given only the set of all closed ST sequential patterns. In the following section of this work, we theoretically and experimentally compare the proposed constricted ST sequential patterns to the closed ST sequential patterns. Thus, Definition~\ref{Def:ClosedST} recalls the notions of closed ST sequential pattern, closure of an ST sequential pattern and PI-strong closed ST sequential pattern. 

\begin{definition}[Closed ST sequential pattern and closure of an ST sequential pattern \citep{ref1284:Maciag2019-Kes}] 
    ST sequential pattern $ \overrightarrow{s_1} $ is closed if there exists no proper supersequence $ \overrightarrow{s_2} $ of $ \overrightarrow{s_1} $, such that the participation index $ PI(\overrightarrow{s_2}) = PI(\overrightarrow{s_1}) $.
    
    A closure of ST sequential pattern $ \overrightarrow{s_1} $ is a supersequence $ \overrightarrow{s_2} $ of $ \overrightarrow{s_1} $, such that $ \overrightarrow{s_2}$ is a closed ST sequential pattern and $ PI(\overrightarrow{s_2}) = PI(\overrightarrow{s_1}) $. 
    \label{Def:ClosedST}
\end{definition}

\begin{definition}
    A PI-strong closed ST sequential pattern is a closed ST sequential pattern whose participation index is greater than the threshold $\theta$.
\label{Def:PIStrongClosedST}
\end{definition}

For example, for the set of ST sequential patterns of Table~\ref{Table:PatternsSet}, $A \rightarrow B \rightarrow B \rightarrow C \rightarrow E \rightarrow C(0.25)$ is a closed ST sequential pattern. $A \rightarrow B \rightarrow B \rightarrow C \rightarrow E \rightarrow C(0.25)$ is also a closure of the following patterns:
\begin{itemize}
    \item $A \rightarrow B \rightarrow B(0.25)$,
    \item $ C \rightarrow E \rightarrow C(0.25)$,
    \item $ A \rightarrow B \rightarrow B \rightarrow C(0.25)$,
    \item $ B \rightarrow C \rightarrow E \rightarrow C(0.25)$,
    \item $ A \rightarrow B \rightarrow B \rightarrow C \rightarrow E(0.25)$,
    \item $ A \rightarrow B \rightarrow C \rightarrow E \rightarrow C(0.25)$,
    \item $ B \rightarrow B \rightarrow C \rightarrow E \rightarrow C(0.25)$.
\end{itemize}

An ST sequential pattern can be closed and be its own closure. For example, $ B \rightarrow B(0.625) $ is a closed ST sequential pattern and it is also its own closure. Please note that according to Definition~\ref{Def:ClosedST} an ST sequential pattern can have more than one closure. For example, pattern $B \rightarrow C \rightarrow E(0.375)$ has two closures:
\begin{itemize}
    \item $ A \rightarrow B \rightarrow C \rightarrow E(0.375)$,
    \item $ B \rightarrow B \rightarrow C \rightarrow E(0.375)$.
\end{itemize}

\section{Discovery of Constricted ST Sequential Patterns}
\label{sec:ConstrictedMaximal}

In this section, we first present our motivation for introducing Constricted ST Sequential patterns. Next, we provide the elementary notions of such type of patterns. 

\subsection{Motivation}

Closed ST sequential patterns of Definition~\ref{Def:ClosedST} are a concise representation of all ST sequential patterns (the set of closed ST sequential patterns can be used to derive all ST sequential patterns). However, in the case of real-world spatio-temporal event data, participation indexes of ST sequential patterns strictly depend on the specification of the neighbourhoods of event instances as well as spatial and temporal distribution of the locations of event instances in the dataset $\mathbf{D}$. Specifically, for a given ST sequential pattern $\overrightarrow{s}$ of length $k$, its proper supersequence patterns of length greater than $k$ usually have only slightly smaller values of participation indexes than $\overrightarrow{s}$. According to Definition~\ref{Def:ClosedST}, none of such supersequences can constitute a closure pattern of $\overrightarrow{s}$ and thus be a closed ST sequential pattern. 
Nevertheless, providing only supersequence patterns of $\overrightarrow{s}$ often will allow us to approximate the participation index of $\overrightarrow{s}$.


Hence, in this paper, we offer a notion of a constricted ST sequential pattern. We define a constricted ST sequential pattern $\overrightarrow{s_1}$ as such maximal (that is the longest) supersequence of pattern $\overrightarrow{s}$ for which (i) the difference between participation indexes of $\overrightarrow{s}$ and $\overrightarrow{s_1}$ is minimal and (ii) participation index of $\overrightarrow{s_1}$ is greater than or equal to the participation index of $\overrightarrow{s}$ minus approximation margin $\varepsilon$\footnote{As we show in Section~\ref{sec:TheoreticalProperties}, for approximation margin $\varepsilon = 0$, the proposed notion of a constricted ST sequential pattern is equivalent to the notion of a closed ST sequential pattern.}. We show in Section~\ref{sec:TheoreticalProperties} that given the set of PI-strong constricted ST sequential patterns, one can obtain the set of all PI-strong ST sequential patterns and approximate participation indexes of each of them with an approximation margin $\pm~\varepsilon$.



\subsection{Elementary Notions}

Let us begin with the definitions of a maximal supersequence of an ST sequential pattern and a minimal proper supersequence of an ST sequential pattern.

\begin{definition}[Maximal supersequence of an ST sequential pattern]
For an ST sequential pattern $\overrightarrow{s_1}$ of a dataset $\mathbf{D}$, its maximal ST supersequnce pattern $\overrightarrow{s_2}$ is such a supersequence of $\overrightarrow{s_1}$ whose length is the greatest. Please note that $\overrightarrow{s_1}$ can have more than one maximal ST supersequence pattern.
\label{Def:MaximalSuperseq}
\end{definition}

\begin{definition}[Minimal proper supersequence of an ST sequential pattern]
For an ST sequential pattern $\overrightarrow{s_1}$ of a dataset $\mathbf{D}$, its minimal proper supersequnce pattern $\overrightarrow{s_2}$ is such a proper supersequence of $\overrightarrow{s_1}$ whose length is the smallest. Please note that $\overrightarrow{s_1}$ can have more than one minimal proper ST supersequence patterns.
\label{Def:MinimalSuperseq}
\end{definition}

The set of supersequences of the pattern $\overrightarrow{s_1} = B \rightarrow B$ presented in Table~\ref{Table:PatternsSet} is: 
\begin{itemize}
    \item $B \rightarrow B$,
    \item $A \rightarrow B \rightarrow B$, 
    \item $B \rightarrow B \rightarrow C$,
    \item $B \rightarrow B \rightarrow D$, 
    \item $A \rightarrow B \rightarrow B \rightarrow C$,
    \item $B \rightarrow B \rightarrow C \rightarrow E$,
    \item $A \rightarrow B \rightarrow B \rightarrow C \rightarrow E$,
    \item $A \rightarrow B \rightarrow B \rightarrow C \rightarrow E \rightarrow C$.
\end{itemize}
From which, the maximal supersequence of the pattern $\overrightarrow{s_1}$ is $A \rightarrow B \rightarrow B \rightarrow C \rightarrow E \rightarrow C$ and the minimal proper supersequences of $\overrightarrow{s_1}$ are $A \rightarrow B \rightarrow B$, $B \rightarrow B \rightarrow C$, $B \rightarrow B \rightarrow D$.

\begin{definition}[$\varepsilon$-constricted maximal supersequence of an ST sequential pattern]
$\varepsilon$-constricted maximal supersequence $\overrightarrow{s_1}$ of an ST sequential pattern $ \overrightarrow{s} $ (in brief, Constricted ST Sequential pattern, CSTS pattern) is such a supersequence of $\overrightarrow{s}$ which preserves the two conditions:
\begin{enumerate}
    \item $PI(\overrightarrow{s_1}) \geq PI(\overrightarrow{s}) - \varepsilon $, and
    \item the difference $PI(\overrightarrow{s}) - PI(\overrightarrow{s_1}$) is minimal over the set of all maximal supersequences of $\overrightarrow{s_1}$.
\end{enumerate}
We denote the set of all $\varepsilon$-constricted maximal supersequences of pattern $\overrightarrow{s}$ as $\mathcal{C}^{max}(\overrightarrow{s})$. $\varepsilon$ is an approximation margin parameter, whose value is user-specified and is in the range [0,1].
\label{Def:ConstrictedSuperseq}
\end{definition}

To illustrate Definition~\ref{Def:ConstrictedSuperseq}, let us consider again the pattern $\overrightarrow{s} = B \rightarrow B$ presented in Table~\ref{Table:PatternsSet} and let us assume that $\varepsilon = 0.25$. The participation index of $\overrightarrow{s}$ equals $0.675$. The $\varepsilon$-constricted supersequence of $\overrightarrow{s}$ is the ST sequential pattern $\mathcal{C}^{max}(\overrightarrow{s}) = \{B \rightarrow B \rightarrow C \rightarrow E\}$, whose participation index equals $0.325$. 

Please note that an ST sequential pattern can have more than one $\varepsilon$-constricted maximal supersequence. For example, let us consider ST sequential pattern $\overrightarrow{s} = B \rightarrow C \rightarrow E$ presented in Table~\ref{Table:PatternsSet} and let us assume that approximation margin $\varepsilon = 0.1$. The two $\varepsilon$-constricted maximal supersequences of $\overrightarrow{s}$ are: $ \mathcal{C}^{max}(\overrightarrow{s}) = \{A \rightarrow B \rightarrow C \rightarrow E, B \rightarrow B \rightarrow C \rightarrow E\}$, whose $PI$ values are both equal to $0.375$. Please also note that according to Definition~\ref{Def:ConstrictedSuperseq} an ST sequential pattern $\overrightarrow{s}$ can be its own $\varepsilon$-constricted maximal supersequence \footnote{In fact, as we present in Section~\ref{sec:CSTS-Miner} each ST sequential pattern of the maximal length patterns set (say $L_k$) is its own $\varepsilon$-constricted maximal supersequence}.

From Definition~\ref{Def:ConstrictedSuperseq} follows that an ST sequential pattern $\overrightarrow{s}$ can be a CSTS pattern of the other ST sequential pattern, but also can have its own CSTS pattern. For example, let us consider the pattern $\overrightarrow{s} = B \rightarrow B \rightarrow C \rightarrow E$ from Table~\ref{Table:PatternsSet} and assume that $\varepsilon = 0.25$. $\overrightarrow{s}$ is a CSTS pattern of the ST sequential pattern $B \rightarrow B$, but also has its own CSTS pattern $\mathcal{C}^{max}(\overrightarrow{s}) = \{A \rightarrow B \rightarrow B \rightarrow C \rightarrow E \rightarrow C\}$.

\begin{definition}
By $\mathcal{RC}^{max}(\overrightarrow{s_1})$ (reverse maximal closure set) we denote a set of ST sequential patterns for which $\overrightarrow{s_1}$ is the $\varepsilon$-constricted maximal supersequence ($\overrightarrow{s_1}$ is a CSTS pattern). 
\label{Def:RCList}
\end{definition}

Let us consider an ST sequential pattern $\overrightarrow{s_1} = A \rightarrow B \rightarrow B \rightarrow C \rightarrow E \rightarrow C$ from Table~\ref{Table:PatternsSet}, whose $PI = 0.25$. Additionally, let us assume that approximation margin is $\varepsilon = 0.25$. The set $\mathcal{RC}^{max}(\overrightarrow{s_1})$ is (the numbers in parentheses specify participation indexes):
\begin{itemize}
    \item $A \rightarrow B \rightarrow B \rightarrow C \rightarrow E \rightarrow C(0.25)$,
    \item $A \rightarrow B \rightarrow B \rightarrow C \rightarrow E(0.25)$, 
    \item $ B \rightarrow B \rightarrow C \rightarrow E \rightarrow C(0.25)$, 
    \item $ A \rightarrow B \rightarrow B \rightarrow C(0.25)$, 
    \item $ B \rightarrow B \rightarrow C \rightarrow E(0.375)$, 
    \item $ B \rightarrow C \rightarrow E \rightarrow C(0.25)$, 
    \item  $A \rightarrow B \rightarrow B(0.25)$, 
    \item $ B \rightarrow B \rightarrow C(0.5)$, 
    \item $ B \rightarrow C \rightarrow E(0.375)$,
    \item  $A \rightarrow B(0.5)$, $ B \rightarrow C(0.5) $, $ E \rightarrow C(0.5) $.
\end{itemize}

The set $\mathcal{RC}^{max}(\overrightarrow{s_1})$ is applied by the proposed CSTS-Miner algorithm to identify CSTS patterns.

\begin{definition}[The set of all PI-strong CSTS patterns] 
The set of all PI-strong CSTS patterns is defined as the set of all ST sequential patterns that are PI-strong (according to Definition~\ref{Def:PIstrong}) and are CSTS patterns (according to Definition~\ref{Def:ConstrictedSuperseq}). 
\end{definition}

The proposed CSTS-Miner algorithm first discovers all PI-strong ST sequential patterns and then subsequently returns only those of them which are PI-strong CSTS patterns. 

\section{Theoretical Properties of $\varepsilon$-constricted Maximal ST Sequential Patterns}
\label{sec:TheoreticalProperties}

In this section, we derive theoretical properties of the introduced notions of (PI-strong) CSTS patterns. Specifically, we show that:
\begin{itemize}
    \item for the parameter $\varepsilon = 0$ the set of all CSTS patterns is equivalent to the set of all closed ST sequential patterns (Lemma~\ref{Lem:CSTSasCLosed});
    \item each CSTS pattern is a closed ST sequential pattern regardless of the value of the approximation margin $\varepsilon $ (Lemma~\ref{Lem:CSTSisClosed});
    \item the set of CSTS patterns is a subset of the set of closed ST patterns for any value of the approximation margin parameter $\varepsilon$ (Theorem~\ref{Thm:CSTSsubsetClosed});
    \item it can be verified if an ST sequential pattern $\overrightarrow{s}$ is PI-strong given the set of PI-strong CSTS patterns and how to approximate $PI$ value of $\overrightarrow{s}$ (Lemma~\ref{Thm:ApproxPIofSTPatt});
    \item one can approximate the value of $PI$ of $\overrightarrow{s}$ with an error no greater than $\pm\varepsilon$ (Lemma~\ref{Lem:MaximalApprox}).
\end{itemize}

Lemma~\ref{Lem:CSTSasCLosed} shows that given an ST sequential pattern $\overrightarrow{s}$ and its CSTS pattern $\overrightarrow{s_1}$ as well as for $\varepsilon = 0$, $\overrightarrow{s_1}$ is a closure of $\overrightarrow{s}$ and $\overrightarrow{s_1}$ is a closed ST sequential pattern.


\begin{lemma} 
     Let $\overrightarrow{s}$ be an ST sequential pattern, $\varepsilon = 0$ and $\overrightarrow{s_1}$ be a CSTS pattern of $\overrightarrow{s}$. The CSTS pattern $\overrightarrow{s_1}$ is a closure of $\overrightarrow{s}$ and $\overrightarrow{s_1}$ is a closed ST sequential pattern.
     \label{Lem:CSTSasCLosed}
\end{lemma}

\begin{proof}
    The proof of lemma follows from Definitions~\ref{Def:ClosedST} and~\ref{Def:ConstrictedSuperseq}. By Definition~\ref{Def:ClosedST}, a closure of ST sequential pattern $\overrightarrow{s}$ is such a closed ST sequential pattern $\overrightarrow{s*}$ that is a supersequence of $\overrightarrow{s}$ and whose $PI(\overrightarrow{s*}) = PI(\overrightarrow{s})$. For $\varepsilon = 0$ we have:
    \begin{itemize} 
    \item the $PI$ value of $\overrightarrow{s_1}$ being a CSTS pattern of $\overrightarrow{s}$ equals $PI$ value of $\overrightarrow{s}$,
    \item CSTS pattern $\overrightarrow{s_1}$ is a maximal supersequence of $\overrightarrow{s}$.
    \end{itemize}
    Thus, for $\varepsilon = 0$ the CSTS pattern $\overrightarrow{s_1}$ is a closed ST sequential pattern.
\end{proof}

It follows from Lemma~\ref{Lem:CSTSasCLosed} that for the parameter $\varepsilon = 0$ the set of all PI-strong CSTS patterns is equivalent to the set of all PI-strong closed ST sequential patterns (in other words, for $\varepsilon = 0$ the proposed algorithm CSTS-Miner will return the same patterns set as the CST-SPMiner of \citep{ref1284:Maciag2019-Kes}). 

In fact, in Lemma~\ref{Lem:CSTSisClosed} we show that each CSTS pattern is a closed ST sequential pattern regardless of the value of the approximation margin $\varepsilon $.

\begin{lemma} 
     Each CSTS pattern is a closed ST sequential pattern regardless of the value of approximation margin $\varepsilon $.
     \label{Lem:CSTSisClosed}
\end{lemma}

\begin{proof}
    Let us assume that $\overrightarrow{s_1}$ is a CSTS pattern of an ST sequential pattern $\overrightarrow{s}$ given any value of $\varepsilon$. Now, let us assume that there exists a proper maximal supersequence $\overrightarrow{s_2}$ of $\overrightarrow{s_1}$ whose participation index $PI(\overrightarrow{s_2})$ equals the participation index $PI(\overrightarrow{s_1})$, that is $\overrightarrow{s_2}$ is a closure of $\overrightarrow{s_1}$. By Definition~\ref{Def:ClosedST} $\overrightarrow{s_2}$ has to be a closed ST sequential pattern. However, as follows from Definition~\ref{Def:ConstrictedSuperseq} this would contradict that $\overrightarrow{s_1}$ is a CSTS pattern. Hence, CSTS pattern $\overrightarrow{s_1}$ is a closed ST sequential pattern.
\end{proof}

In Theorem~\ref{Thm:CSTSsubsetClosed}, we show that the set of CSTS patterns is a subset of the set of closed ST patterns for any value of the approximation margin  parameter $\varepsilon$.

\begin{theorem}
     For any $\varepsilon \in [0,1]$, the set of all CSTS patterns is a subset of the set of all closed ST sequential patterns. 
     \label{Thm:CSTSsubsetClosed}
\end{theorem}

\begin{proof}
    We already showed in Lemma~\ref{Lem:CSTSasCLosed} that for $\varepsilon = 0$, the set of all CSTS patterns is equal to the set of all closed ST sequential patterns. We also presented in Lemma~\ref{Lem:CSTSisClosed} that each CSTS pattern is a closed ST sequential pattern regardless of $\varepsilon$ value. Let us assume that $\varepsilon > 0$. If there exists a closed ST sequential pattern $\overrightarrow{s}$ that has a CSTS pattern and $\overrightarrow{s}$ is not a CSTS pattern itself, then $\overrightarrow{s}$ will not be included in the set of CSTS patterns (in such a case, the CSTS patterns set is a proper subset of the set of closed ST sequential patterns). Otherwise, if each closed ST sequential pattern is also a CSTS pattern, then the set of CSTS patterns is equal to the set of all closed ST sequential patterns. In any case, the CSTS patterns set is a subset of the closed ST sequential patterns set.
\end{proof}

Theorem~\ref{Thm:CSTSsubsetClosed} shows that the number of the discovered CSTS patterns is less than or equal to the number of closed ST sequential patterns. Thus, by discovering of only CSTS patterns one can always obtain as many as all closed ST sequential patterns. 

Lemma~\ref{Thm:ApproxPIofSTPatt} shows how to verify if an ST sequential pattern $\overrightarrow{s}$ is PI-strong given the set of PI-strong CSTS patterns and how to approximate its $PI$ value. 

\begin{lemma} For each ST sequential pattern holds:
\begin{enumerate}[start=1,label=(\roman*).]
    \item An ST sequential pattern $\overrightarrow{s}$ is PI-strong only if there exists a supersequence of $\overrightarrow{s}$ in the set of PI-strong CSTS patterns. 
    \item The participation index value of a PI-strong ST sequential pattern $\overrightarrow{s}$ is equal to or less than $PI(\overrightarrow{s_1}) + \varepsilon$, where $\overrightarrow{s_1}$ is such a minimal proper supersequence of $\overrightarrow{s}$ in PI-strong CSTS patterns, whose $PI(\overrightarrow{s_1})$ is the greatest.
\end{enumerate}
\label{Thm:ApproxPIofSTPatt}
\end{lemma}

\begin{proof}~~~~~~~~~~~
    \begin{enumerate}[start=1,label=Ad (\roman*)., ,leftmargin = 4em]
    \item Follows immediately from Theorem~\ref{theorem:1}. 
    \item  \textbf{Case 1.} Let us first assume that there is only one supersequence $\overrightarrow{s_1}$ of $\overrightarrow{s}$ in the set of PI-strong CSTS patterns. In this case, $\overrightarrow{s_1}$ is a $\varepsilon$-constricted maximal ST supersequence of $\overrightarrow{s}$ and by Definition~\ref{Def:ConstrictedSuperseq} $PI(\overrightarrow{s}) \in \Big[PI(\overrightarrow{s_1}), PI(\overrightarrow{s_1}) + \varepsilon$\Big]. \\
    \textbf{Case 2.} Now let us assume that there is more than one proper supersequence of $\overrightarrow{s}$ in the set of PI-strong CSTS patterns. Since we can not indicate which one of them is $\varepsilon$-constricted maximal ST supersequence, then $PI(\overrightarrow{s}) \in \Big[PI(\overrightarrow{s_1}), PI(\overrightarrow{s_1}) + \varepsilon$\Big], where $\overrightarrow{s_1}$ is a minimal proper supersequence from all proper supersequences of $\overrightarrow{s}$ in PI-strong CSTS patterns.
\end{enumerate}
\end{proof}

Lemma~\ref{Thm:ApproxPIofSTPatt} indicates that the set of PI-strong CSTS patterns (unlike the set of all PI-strong closed ST sequential patterns) is not a lossless representation of all PI-strong ST sequential patterns in a sense that the PI value of a PI-strong ST sequential pattern can be obtained from the set of PI-strong CSTS patterns only with a certain approximation. The question is how much the approximation of such PI value of a PI-strong ST sequential pattern differs from its exact PI value. In Lemma~\ref{Lem:MaximalApprox}, we provide the value of such maximal approximation.

\begin{lemma}
     Given an ST sequential pattern $\overrightarrow{s}$ that has a supersequence in the CSTS patterns set one can not:
     \begin{itemize}
         \item underestimate the exact $PI$ value of $\overrightarrow{s}$ by less than $PI(\overrightarrow{s}) - \varepsilon$, and
         \item overestimate the exact $PI$ value of $\overrightarrow{s}$ by more than $PI(\overrightarrow{s}) + \varepsilon$.
     \end{itemize}  
     
     
     \label{Lem:MaximalApprox}
\end{lemma}

\begin{proof}
    Let us assume that $\overrightarrow{s_1}$ is any proper superseuqence of $\overrightarrow{s}$ in the set of CSTS patterns. We will consider the two extreme cases:\\
    \textbf{Case 1.} The value of $PI(\overrightarrow{s_1})$ equals $PI(\overrightarrow{s}) - \varepsilon$. In such a case, one can not underestimate $PI(\overrightarrow{s})$ by less than $PI(\overrightarrow{s}) - \varepsilon$. \\
    \textbf{Case 2.} The value of $PI(\overrightarrow{s_1})$ equals $PI(\overrightarrow{s})$. In such a case, one can not overestimate $PI(\overrightarrow{s})$ by more than $PI(\overrightarrow{s}) + \varepsilon$.\\
    Thus, the estimation of the $PI(\overrightarrow{s})$ value of the pattern $\overrightarrow{s}$ given the set of CSTS patters is always in the range $\Big[PI(\overrightarrow{s}) - \varepsilon, PI(\overrightarrow{s}) + \varepsilon \Big]$.
\end{proof}

\section{The Constricted ST Sequential Patterns Miner}
\label{sec:CSTS-Miner}

This section introduces our algorithm called CSTS-Miner for discovering the set of all PI-strong CSTS patterns. In Table~\ref{Table:Notation}, we present the notation used in the algorithms of this section. CSTS-Miner is presented in Algorithm~\ref{Alg:CSTS-Miner}. The algorithm consists of two phases: i. "top-down" - iterative generation of all PI-strong ST sequential patterns of length $k$ from patterns of length $k - 1$ until it is impossible to generate new patterns; ii. "bottom-up" - calculation of PI-strong CSTS patterns. The "top-down" phase adapts the previously introduced STBFM algorithm for the discovery of PI-strong ST sequential patterns. Specifically, to efficiently generate new PI-strong patterns of length $k$ from PI-strong patterns of length $k-1$, we adapted the SP-tree structure and extended it to the proposed MAX-Tree structure. MAX-Tree is used to not only iteratively generate new patterns but also, unlike SP-Tree offered in \citep{ref1284:Maciag2019}, to identify all PI-strong CSTS patterns. The "bottom-up" phase calculates PI-strong CSTS patterns in a recursive way starting with the set of the longest PI-strong ST sequential patterns $L_k$ obtained in the "top-down" phase.

\begin{table}[h!t]
\caption{Notations and parameters used in the algorithms.}
\small{
\resizebox{\columnwidth}{!}{\begin{tabularx}{\linewidth}{ll}
	\toprule
    \textbf{Notation} & \textbf{Description} \\
    \midrule
    $\mathbf{D}$ & A dataset of spatio-temporal event instances \\
    $\mathbf{F}$ & A set of event types \\
    $\overrightarrow{s}$ & An ST sequential pattern \\
    $\mathbf{N}(e, F)$ & \makecell[l]{Neighborhood of event instance $e$ \\with respect to event type $F \in \mathbf{F}$} \\
    $\mathbf{I}(\overrightarrow{s},i)$ &  \makecell[l]{A set of event instances supporting \\i-th element of $\overrightarrow{s}$} \\
    $PR(\overrightarrow{s}, i)$ & \makecell[l]{Participation ratio of \\i-th element of $\overrightarrow{s}$} \\
    $PI(\overrightarrow{s})$ & Participation index of pattern $\overrightarrow{s}$  \\
    $\text{children}(\overrightarrow{s}), \overrightarrow{s}.\text{children}$ & Children patterns of $\overrightarrow{s}$ \\
    $\text{parent}_1(\overrightarrow{s}), \overrightarrow{s}.\text{parent}_1$ & First parent of pattern $\overrightarrow{s}$ \\
    $\text{parent}_2(\overrightarrow{s}), \overrightarrow{s}.\text{parent}_2$ & Second parent of pattern $\overrightarrow{s}$ \\
    $\mathcal{C}^{max}(\overrightarrow{s}), \overrightarrow{s}.\mathcal{C}^{max}$ & \makecell[l]{A set of $\varepsilon$-constricted maximal \\ supersequences of $\overrightarrow{s}$}\\
    $\mathcal{RC}(\overrightarrow{s}), \overrightarrow{s}.\mathcal{RC}$ & \makecell[l]{Set of ST sequential patterns, \\for which $\overrightarrow{s}$ is a $\varepsilon$-constricted\\ maximal supersequence}\\
    $\theta$ & \makecell[l]{Participation index threshold of \\discovered patterns}\\
    $\varepsilon$ & \makecell[l]{Approximation margin}\\
	\bottomrule
\end{tabularx}}
}
\label{Table:Notation}
\end{table}

\subsection{The "Top-down" Phase of the CSTS-Miner Algorithm}

The "top-down" phase of Algorithm~\ref{Alg:CSTS-Miner} is conducted by executing steps 1 to 18 of this algorithm. In step 1 of Algorithm~\ref{Alg:CSTS-Miner}, a singular candidate ST sequential pattern is generated from each event type in $\mathbf{F}$ and remembered in the set $L_1$ (patterns in $L_1$ constitute the first level of MAX-Tree). By Definitions~\ref{Def:PR} and~\ref{Def:PI}, singular patterns are always PI-strong since their PI values are equal to 1. 

Subsequently, the PI-strong ST sequential patterns of length $2$ are generated and remembered as $L_2$. The generation of such PI-strong ST sequential patterns is conducted in steps 3 to 13 using two nested loops, each of which iterates over all patterns in $L_1$. A new candidate pattern $\overrightarrow{s}$ of length 2 is always generated by concatenating the two event types of singular patterns $\overrightarrow{s_i}$ and $\overrightarrow{s_j}$. We will refer to the patterns $\overrightarrow{s_i}$ and $\overrightarrow{s_j}$ as the first and the second parent of $\overrightarrow{s}$, respectively. Please note that $\overrightarrow{s}$ can consist of two the same event types (in such a case, $\overrightarrow{s_i} = \overrightarrow{s_j}$). The set of instances supporting the second element of $\overrightarrow{s}$ is calculated in step 6 and consists of event instances of type $\overrightarrow{s}[2]$ in $\mathbf{D}$ which belong to neighborhoods of event instances in the set $\mathbf{I}(\overrightarrow{s_i}, 1)$. 

The participation index of the candidate generated pattern $\overrightarrow{s}$ is calculated and verified in steps 7 and 8 of Algorithm~\ref{Alg:CSTS-Miner}. If $\overrightarrow{s}$ occurred to be PI-strong, then $\overrightarrow{s}$ is inserted into the list of children of its first parent $\text{children}(\overrightarrow{s_i})$ and into the set $L_2$.

\begin{algorithm}[h!t]
\small{
	\caption{\protect \Call{CSTS-Miner}{$\mathbf{D}, \mathbf{F}, R, T, \theta, \varepsilon$}}
	\begin{algorithmic}[1]
		\Require $ \mathbf{D} $ - spatio-temporal event dataset, $ \mathbf{F} $ - set of event types, $ R $ - spatial threshold value, $ T $ - time window threshold value, $ \theta $ - participation index threshold value, $\varepsilon$ - approximation margin.
		\Statex \textbf{Assumption:} $L_k$ - a set of PI-strong ST sequential patterns of length $k$.
	    \Ensure - $\bigcup_{k} \big\{ \overrightarrow{s} \in L_k \text{ | } \overrightarrow{s}.\mathcal{RC} \neq \emptyset \lor (\overrightarrow{s}.\mathcal{RC} = \emptyset \land \overrightarrow{s}.\mathcal{C}^{max} = \emptyset)\}$ - a set of PI-strong CSTS patterns.
		\Statex
	    \Statex \{*------------- CSTS-Miner "top-down" phase -------------*\}
		\State $ L_1 :=$ generate patterns of length 1 from all event types in $\mathbf{F}$
		\State $ L_2 := \emptyset$ 
		\For {\textbf{each} $\overrightarrow{s_i} \in L_{1} $} \label{step:L2Begin}
		\For {\textbf{each} $\overrightarrow{s_j} \in L_{1} $}
		\State $\overrightarrow{s} := \overrightarrow{s_i}[1] \rightarrow \overrightarrow{s_j}[1]$ 
		\State $\mathbf{I}(\overrightarrow{s}, 2) :=  \bigcup\limits_{e \in \mathbf{I}(\overrightarrow{s_i}, 1)} N_{\overrightarrow{s}[2]}(e)$
		\State $\overrightarrow{s}.PI := PR(\overrightarrow{s}, 2)$
		\If{$\overrightarrow{s}.PI > \theta$}
		\State $ \overrightarrow{s}.\text{parent}_1 $ := $ \overrightarrow{s_i} $; $ \overrightarrow{s}.\text{parent}_2 $ := $ \overrightarrow{s_j} $
		\State Add $ \overrightarrow{s} $ to $ \overrightarrow{s}.\text{parent}_1.\text{children} $; add $ \overrightarrow{s} $ to $ L_2 $ \label{step:AddtoL2}
		\EndIf
		\EndFor
		\EndFor \label{Step:L2End}
		
		\State $ k := 2 $
		\While {$ L_{k} \ne \emptyset$}
		\State $ k := k + 1$
		\State $ L_{k} :=$ \Call{GenAndVerify}{$ L_{k-1} $} 
		\EndWhile
		
		\Statex
	    \Statex \{*------------- CSTS-Miner "bottom-up" phase -------------*\}
		\While {$ k > 1$}
		    \For{$\overrightarrow{s} \in L_{k}$}
		        \State \Call{VerifySupersequence}{$\protect \overrightarrow{s}, \protect \overrightarrow{s}.\text{parent}_1 $}
		        \State \Call{VerifySupersequence}{$\protect \overrightarrow{s}, \protect \overrightarrow{s}.\text{parent}_2 $}
		    \EndFor 
		\State $ k := k - 1$
		\EndWhile

		\State \Return $\bigcup_{k} \big\{ \overrightarrow{s} \in L_k \text{ | } \overrightarrow{s}.\mathcal{RC}^{max} \neq \emptyset \lor (\overrightarrow{s}.\mathcal{RC}^{max} = \emptyset \land \overrightarrow{s}.\mathcal{C}^{max} = \emptyset)\} $ \Comment{Return all CSTS patterns.}
	\end{algorithmic}
		\label{Alg:CSTS-Miner}
	}
\end{algorithm}

\begin{algorithm}[h!t]
\small{
	\caption{\protect \Call{GenAndVerify}{$ L_{k-1} $}}
	\begin{algorithmic}[1]
		\Require $ L_{k - 1} $ - a set of PI-strong ST sequential patterns of length $ k - 1 $.
		\Ensure $ L_{k} $ - a set of PI-strong ST sequential patterns of length $ k $.
		\State $ L_k := \emptyset $
		\For {\textbf{each} $\overrightarrow{s_i} \in L_{k-1} $} \label{line1}
		\State $\overrightarrow{s_l} := \overrightarrow{s_i}.\text{parent}_2$
		\For {\textbf{each} $ \overrightarrow{s_j} \in \overrightarrow{s_l}.\mathcal{X}$ } \label{line2}
		\State 
		$ \overrightarrow{s} := \overrightarrow{s_i}[1] \rightarrow \overrightarrow{s_i}[2] \rightarrow$ $\dots \rightarrow$ 
		        $\overrightarrow{s_i}[k - 1] \rightarrow  \overrightarrow{s_j}[k - 1] $
		\State $\mathbf{I}(\overrightarrow{s}, k) := \bigcup\limits_{e \in \mathbf{I}(\overrightarrow{s_i},k-1)} N_{\overrightarrow{s}[k]}(e)$ 
		\State $\overrightarrow{s}.PI := $ min$\big(\overrightarrow{s_i}.PI, PR(\overrightarrow{s}, k)\big)$ 
		\If{$\overrightarrow{s}.PI > \theta$}
		\State $ \overrightarrow{s}.\text{parent}_1 $ := $ \overrightarrow{s_i} $; $ \overrightarrow{s}.\text{parent}_2 $ := $ \overrightarrow{s_j} $
		\State Add $ \overrightarrow{s} $ to $ \overrightarrow{s}.\text{parent}1.\text{children} $; add $ \overrightarrow{s} $ to $ L_k $
		\EndIf
		\EndFor 
		\EndFor
		\State \Return $ L_k $
	\end{algorithmic}
		\label{Alg:GenAndVerify}
	}
\end{algorithm}

Steps 15 to 18 of the Algorithm~\ref{Alg:CSTS-Miner} consists of the iterative generation and verification of PI-strong ST sequential patterns of length greater than 2 using the function $\textit{GenAndVerify}{}$ shown in Algorithm~\ref{Alg:GenAndVerify}. Specifically, the $\textit{GenAndVerify}{}$ function generates all PI-strong ST sequential patterns $L_k$ from PI-strong ST sequential patterns $L_{k-1}$. To this end, the function uses the first parent, second parent as well as the children list of patterns in $L_{k-1}$. As follows from Lemma~\ref{Lem:5} (presented in \citep{ref1284:Maciag2019-Kes} as Lemma 5), a candidate ST sequential pattern $\overrightarrow{s}$ of length $\geq 3$ can be obtained by concatenating all elements of its first parent with the last element of its second parent. \footnote{We refer the reader to \citep{ref1284:Maciag2019-Kes} for the proof of Lemma~\ref{Lem:5}.}


\begin{lemma} [Construction of an ST sequential pattern from its first parent and second parent \citep{ref1284:Maciag2019-Kes}]
     Let $ \overrightarrow{s}$ be an ST sequential pattern of length $m \geq 3$ and $ \overrightarrow{s_1}$ be the first parent of $ \overrightarrow{s}$. Then, $ \overrightarrow{s} = \overrightarrow{s_1} \rightarrow F_m$, where $F_m$ is the last element of $\text{parent}_1(\overrightarrow{s})$, and $\text{parent}_2(\overrightarrow{s}) = \text{parent}_2(\overrightarrow{s_1}) \rightarrow F_m $.
     \label{Lem:5}
\end{lemma}

Algorithm~\ref{Alg:GenAndVerify} proceeds as follows. First the set $L_k$ is initialized. Subsequently, the loop in step 2 iterates over all patterns in $\overrightarrow{s_i} \in L_{k-1}$ ($\overrightarrow{s_i}$ is the first parent of a new candidate pattern $\overrightarrow{s}$) and the loop in step 4 iterates over all children patterns $\overrightarrow{s_j} \in \text{parent}_2(\overrightarrow{s_i}))$ of the second parent of $\overrightarrow{s_i}$. Each of such $\overrightarrow{s_j}$ child patterns constitutes the second parent of a new candidate pattern $\overrightarrow{s}$. Thus, the elements of $\overrightarrow{s}$ are: $ \overrightarrow{s} := \overrightarrow{s_i}[1] \rightarrow \overrightarrow{s_i}[2] \rightarrow$ $\dots \rightarrow \overrightarrow{s_i}[k - 1] \rightarrow  \overrightarrow{s_j}[k - 1] $.

After the elements of $\overrightarrow{s}$ are obtained, the set of instances supporting its last element is computed according to step 6 of Algorithm~\ref{Alg:GenAndVerify} and the $PI$ value of $\overrightarrow{s}$ is calculated according to step 8 of Algorithm~\ref{Alg:GenAndVerify}. If the $PI$ value is greater than the participation index threshold $\theta$, then $\overrightarrow{s}$ is inserted to $L_{k}$ and appended to the children list of $\overrightarrow{s_i}$: $\text{children}(\overrightarrow{s_i})$. Otherwise, $\overrightarrow{s}$ is discarded as not being PI-strong pattern.

The generation of new candidate ST sequential patterns ends when the $\textit{GenAndVerify}{}$ function can not generate any new patterns.

\subsection{The "Bottom-up" Phase of the CSTS-Miner Algorithm}

\begin{algorithm}[h!t]
\small{
	\caption{\protect \Call{VerifySupersequence}{$\protect\overrightarrow{s}, \protect\overrightarrow{s_i}$}}
	\begin{algorithmic}[1]
		\Require $\overrightarrow{s}$ - PI-strong ST sequential pattern, whose $\mathcal{RC}^{max}$ set needs to be obtained.

		\If{$\overrightarrow{s}.PI \geq \overrightarrow{s_i}.PI - \epsilon$} 
		\If{$\overrightarrow{s_i} \notin \overrightarrow{s}.\mathcal{RC}^{max} $}  
		\If{$\overrightarrow{s_i}.\mathcal{C}^{max} = \emptyset \lor \text{len}(\overrightarrow{s}) = \text{len}(\overrightarrow{s_i}.\mathcal{C}^{max})$}
	
		\If{$\overrightarrow{s_i}.\mathcal{C}^{max} = \emptyset \lor \overrightarrow{s}.PI = PI(\overrightarrow{s_i}.\mathcal{C}^{max})$}
		    \State Insert $\overrightarrow{s}$ to $\overrightarrow{s_i}.\mathcal{C}^{max}$
		    \State Insert $\overrightarrow{s_i}$ to $\overrightarrow{s}.\mathcal{RC}^{max}$
		\ElsIf {$\overrightarrow{s}.PI > PI(\overrightarrow{s_i}.\mathcal{C}^{max})$}
		    \State \begin{varwidth}[t]{\linewidth}
		    Remove $\overrightarrow{s_i}$ from \par
		        \hskip\algorithmicindent  $\overrightarrow{s_k}.\mathcal{RC}^{max}$ of \textbf{all} $\overrightarrow{s_k} \in \overrightarrow{s_i}.\mathcal{C}^{max}$
		    \end{varwidth}
		    \vspace{0.1cm}
		    \State $\overrightarrow{s_i}.\mathcal{C}^{max} := \emptyset$
		    \State Insert $\overrightarrow{s}$ to $\overrightarrow{s_i}.\mathcal{C}^{max}$
		    \State Insert $\overrightarrow{s_i}$ to $\overrightarrow{s}.\mathcal{RC}^{max}$
		\EndIf
		
		\EndIf
		\EndIf
		\If{$\overrightarrow{s_i}.\text{parent}_1 \neq$ NULL}
		\State \Call{VerifySupersequence}{$\protect\overrightarrow{s}, \protect \overrightarrow{s_i}.\text{parent}_1$}
		\EndIf
		\If{$\overrightarrow{s_i}.\text{parent}2 \neq$ NULL}
		\State \Call{VerifySupersequence}{$\protect\overrightarrow{s}, \protect \overrightarrow{s_i}.\text{parent}_2$}
		\EndIf

		\EndIf
		\State \Return 
	\end{algorithmic}
	\label{Alg:ClosureCalculate}
	}
\end{algorithm}

While the "top-down" phase generates all PI-strong ST sequential patterns, the "bottom-up" phase of Algorithm~\ref{Alg:CSTS-Miner} starts in step 19 and is entirely dedicated to calculation of these PI-strong ST sequential patterns which are PI-strong CSTS patterns. The verification of PI-strong ST sequential patterns as being CSTS patterns starts from the set $L_k$ of the longest patterns and is iteratively continued for the subsequent sets of patterns $L_{k-1}, L_{k-2}, \dots, L_2$. 


The function $\textit{VerifySupersequence}{}$ presented in Algorithm~\ref{Alg:ClosureCalculate} receives two ST sequential patterns $\overrightarrow{s}$ and $\overrightarrow{s_i}$ and verifies if $\overrightarrow{s}$ is a CSTS pattern of $\overrightarrow{s_i}$. To this end, function $\textit{VerifySupersequence}{}$ applies Definition~\ref{Def:ConstrictedSuperseq} and Theorem~\ref{theorem:1}. Specifically, function $\textit{VerifySupersequence}{}$ conducts the following steps:
\begin{enumerate} 
    \item Checks if the $PI$ value of $\overrightarrow{s}$ is greater than the $PI$ value of $\overrightarrow{s_i}$ minus approximation margin $\varepsilon$ (in step~1). This fulfils the first condition of Definition~\ref{Def:ConstrictedSuperseq}.
    \item Checks if $\overrightarrow{s_i}$ already belongs to the $RC^{max}$ list of $\overrightarrow{s}$ (in step~2). Due to construction of MAX-Tree, it can happen that the function $\textit{VerifySupersequence}{}$ will be invoked multiple times for the same two sequences $\overrightarrow{s}$ and $\overrightarrow{s_i}$. Thus, the check prevents the situation when $\overrightarrow{s_i}$ is added multiple times to the list $RC^{max}(\overrightarrow{s})$.
    \item Verifies whether there is no pattern of $\overrightarrow{s_i}$ in the set $C^{max}(\overrightarrow{s_i})$ or the length of $\overrightarrow{s}$ equals the length of the patterns in $C^{max}(\overrightarrow{s_i})$. In such a case, the two situations are possible:
    \begin{itemize}
        \item Either $C^{max}(\overrightarrow{s_i}) = 0$ or the $PI$ value of $\overrightarrow{s}$ equals the $PI$ values of $C^{max}(\overrightarrow{s_i})$ patterns. In any case,  $\overrightarrow{s}$ is inserted to $C^{max}(\overrightarrow{s_i})$ and $\overrightarrow{s_i}$ is inserted to $\mathcal{RC}^{max}(\overrightarrow{s_i})$. 
        \item Alternatively, if the $PI$ value of $\overrightarrow{s}$ is greater than the $PI$ values of patterns in $C^{max}(\overrightarrow{s_i})$, then $\overrightarrow{s_i}$ is removed from the $\mathcal{RC}^{max}$ lists of all patterns in $C^{max}(\overrightarrow{s_i})$ and $C^{max}(\overrightarrow{s_i})$ is set to be empty. Next, $\overrightarrow{s}$ is added to $C^{max}(\overrightarrow{s_i})$ and $\overrightarrow{s_i}$ is added to $\mathcal{RC}^{max}(\overrightarrow{s})$. This set fulfils the second condition of Definition~\ref{Def:ConstrictedSuperseq}.
    \end{itemize}
\end{enumerate}

If the $PI$ value of $\overrightarrow{s}$ is greater than the $PI$ value of $\overrightarrow{s_i}$ minus approximation margin $\varepsilon$, then $\textit{VerifySupersequence}{}$ function is recursively invoked for $\overrightarrow{s}$ and the parent patterns $\text{parent}_1(\overrightarrow{s_i})$, $\text{parent}_2(\overrightarrow{s_i})$ of $\overrightarrow{s_i}$ to verify whether $\overrightarrow{s}$ is also their CSTS pattern. Otherwise, as follows from Theorem~\ref{theorem:1}, $\overrightarrow{s}$ can not be a CSTS pattern of any of parent patterns $\text{parent}_1(\overrightarrow{s_i})$, $\text{parent}_2(\overrightarrow{s_i})$ of $\overrightarrow{s_i}$. Thus, invokes of $\textit{VerifySupersequence}{\overrightarrow{s}, \text{parent}_1(\overrightarrow{s_i})}$ and $\textit{VerifySupersequence}{\overrightarrow{s}, \text{parent}_2(\overrightarrow{s_i})}$ are skipped. 

\begin{landscape}
\begin{figure*}[h!t]
	\centering 
	\includegraphics[width=1.0\linewidth]{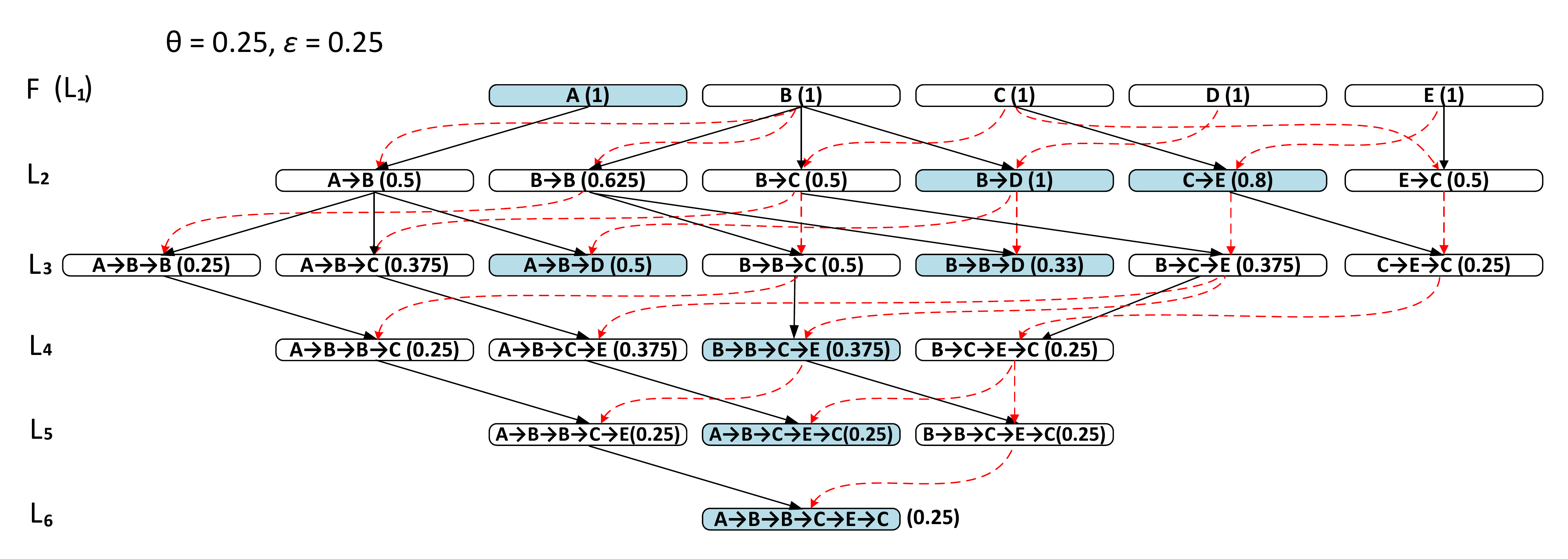}
	\caption{The complete MAX-Tree generated with parameters $\theta = 0.25$ and $\varepsilon = 0.25$ for the dataset presented in Figure~\ref{Fig:Dataset} (all boxes represent PI-strong ST sequential patterns, while blue boxes represent patterns being also PI-strong CSTS patterns; PI values are given in parentheses next to the patterns).}
	\label{Fig:CompleteMAX-Tree}      
\end{figure*}
\end{landscape}

The complete MAX-Tree created for the dataset shown in Figure~\ref{Fig:Dataset} is presented in Figure~\ref{Fig:CompleteMAX-Tree}. The tree is created by Algorithm~\ref{Alg:CSTS-Miner} using the following input parameters: spatial threshold value $R = 10$, time window threshold value $T = 20$, participation index threshold value $\theta = 0.25$, approximation margin value $\varepsilon = 0.25$. All patterns in the tree are PI-strong ST sequential patterns; however, only blue boxes represent those of them, which are also PI-strong CSTS patterns. 

For the pattern $\overrightarrow{s} = B \rightarrow B \rightarrow C$, the set of its CSTS patterns is $\mathcal{C}^{max}(\overrightarrow{s}) = \{A \rightarrow B \rightarrow B \rightarrow C \rightarrow E \rightarrow C\}$, while the minimal proper supersequence of $\overrightarrow{s}$ among the set of PI-strong CSTS patterns returned by Algorithm~\ref{Alg:CSTS-Miner} is $\overrightarrow{s_1} = B \rightarrow B \rightarrow C \rightarrow E$. Let us consider how one can approximate the participation index value of pattern $\overrightarrow{s} = B \rightarrow B \rightarrow C$ given the set of all PI-strong CSTS patterns presented in Figure~\ref{Fig:CompleteMAX-Tree}. Since the set of PI-strong CSTS patterns contains more than one supersequence of $\overrightarrow{s}$, then the value of participation index of $PI(\overrightarrow{s)} \leq PI(\overrightarrow{s_1}) + \varepsilon = 0.375 + 0.25 = 0.625$. In fact, the exact PI value of $\overrightarrow{s}$ equals $ PI(\overrightarrow{s}) = 0.5$.

\section{Experiments}
\label{sec:ExperimentalEvaluation}

In this section, we first review datasets used for the experiments and describe our experimental setup. Then we provide results of the comparison of the proposed CSTS-Miner algorithm with the STS-Miner \citep{ref1284:Huang2008}, CST-SPMiner \citep{ref1284:Maciag2019-Kes}, STBFM \citep{ref1284:Maciag2019-Kes} and CSTPM \citep{ref1284:Mohan2012} algorithms. 

\subsection{Selected Dataset}

For the experiments with the proposed CSTS-Miner algorithm we selected two publicly available datasets, each of which consists of crime event incidents. 

\subsubsection{Pittsburgh Police Incident Blotter Dataset}

The first dataset selected for the experiments is \textit{Pittsburgh Police Incident Blotter Dataset} that contains crime incidents collected by the Police Department of Pittsburgh City over the period 31.12.1989 - 31.12.2019  \citep{ref1284:PittsburghCrimeData}. The dataset was validated according to the Uniform Crime Reporting (UCR) standards \citep{ref1284:UniformCrimeReporting} and consists of such attributes as: \textit{incident time}, \textit{incident location} (defined by street name and number), \textit{incident neighborhood}, \textit{incident type}, \textit{description of offense}, incident \textit{longitude} and \textit{latitude} locations. For the purposes of experiments we selected the following attributes of the dataset: \textit{incident type}, \textit{incident time}, \textit{longitude} and \textit{latitude} which are directly used by the proposed CSTS-Miner algorithm. In the experiments, we use only crime incidents reported between 01.01.2017 and 31.12.2019. The number of crime incidents reported over this period is \num{122895}; however, approximately 40\% of them contain missing values for one of the selected attributes: incident type, geographical location or incident time. Thus, we decided to remove such incidents. The resultant dataset contains \num{72867} crime event incidents of 236 unique incident types. All of these 236 incident types are selected as event type set \textbf{F}. To the most frequent incident types belong: \textit{theft from auto}, \textit{simple assault}, \textit{public drunkenness}, \textit{criminal mischief} or \textit{harassment}. The attribute \textit{incident time} (which specifies incident occurrence time with exact occurrence date and a timestamp given in hours, minutes and seconds) is transformed into the number of minutes that passed from the timestamp \textit{01.01.2017 00:00}. Thus, the time window parameter $T$ used by the CSTS-Miner algorithm is specified in minutes.  In Table~\ref{Table:PittsburgCharacteristic}, we present the characteristic of the resultant dataset.

In Figure~\ref{Fig:PittsburghDataset}, we present the location of the first two thousand crime incidents of ten most frequent incident types from the resultant Pittsburgh Police Incident Blotter Dataset. 

\begin{table}[h!t]
\caption{The characteristic of the resultant Pittsburgh Police Incident Blotter Dataset.}
\resizebox{\columnwidth}{!}{\begin{tabularx}{\linewidth}{lll}
	\toprule
    \textbf{Parameter} & \textbf{Value} \\
    \midrule
    No. of incident types (|\textbf{F}|) & 236 \\
    No. of incident instances (|\textbf{D}|) & 72~867 \\
    Avg. no. of incidents per type & 309 \\
    Median no. of incidents per type & 26 \\
    Std. of no. of incidents per type & 784 \\
    Min. no. of instances in a type & 1 \\
    Max. no. of instances in a type & 6201 \\
	\bottomrule
\end{tabularx}}
\label{Table:PittsburgCharacteristic}
\end{table}

\begin{figure*}[h!t]
	\centering 
	\includegraphics[width=1\textwidth]{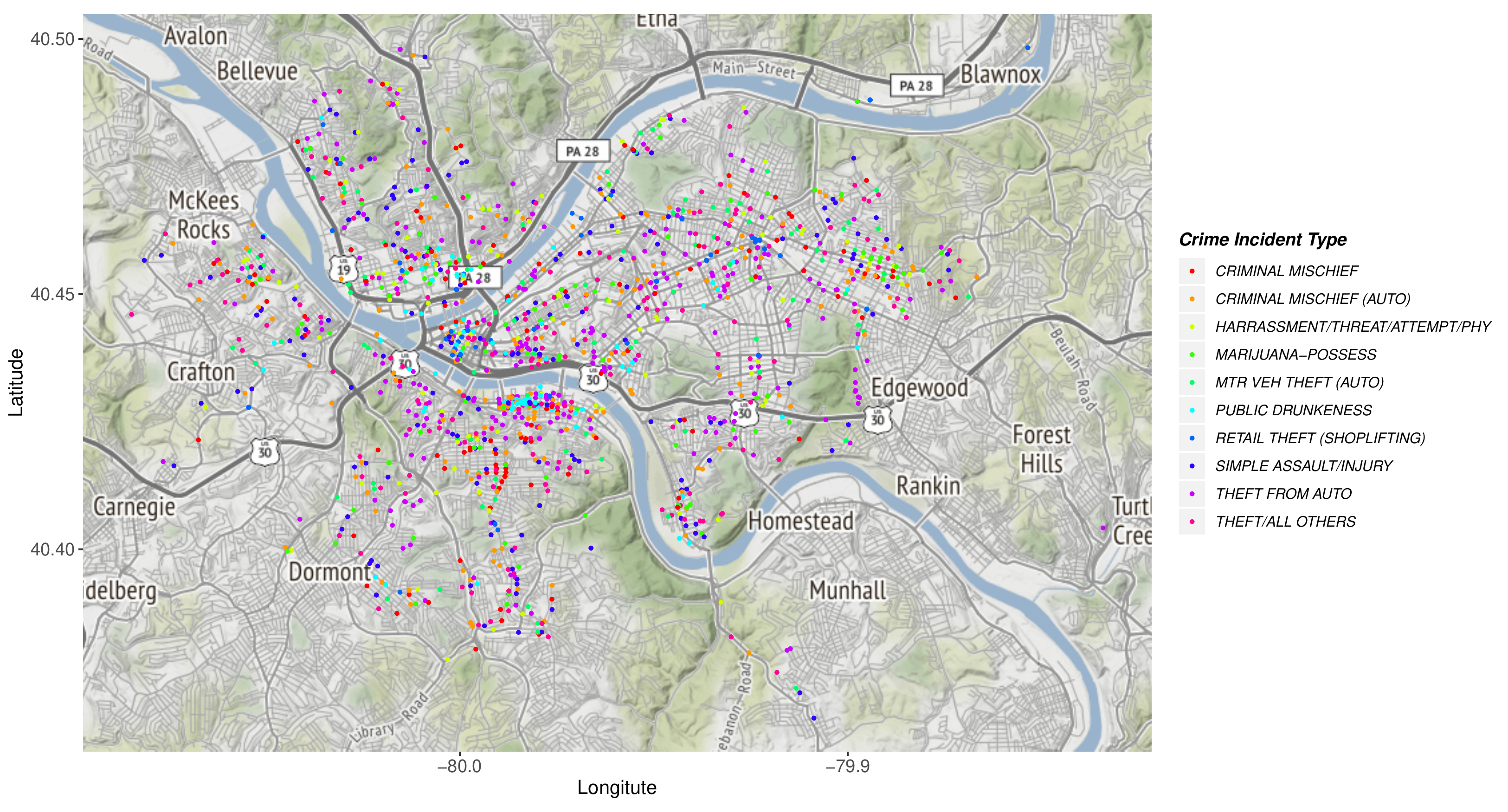}
	\caption{The first two thousand crime incidents of ten most frequent incident types from the resultant Pittsburgh Police Incident Blotter Dataset.}
	\label{Fig:PittsburghDataset}      
\end{figure*}

\subsubsection{Boston Crime Incident Reports Dataset}

The second of the selected datasets is \textit{Boston Crime Incidents Reports Dataset} provided by the Boston Police Department \citep{ref1284:BostonCrimeIncidents}. The dataset was collected over the period 08.07.2012 - 10.08.2015. However, in the experiments, we extracted only crime incidents that occurred between 01.01.2014 - 31.12.2014. The Boston Crime Incidents Reports Dataset contains several attributes, such as: \textit{incident location}, \textit{incident time}, \textit{incident type}, \textit{used weapon type} (such as, for example, unarmed or firearm) \textit{shooting presence}, \textit{police shift} or \textit{occurrence district} and \textit{occurrence area}. Similarly to the Pittsburgh Police Incident Blotter Dataset, for the experiments we selected only attributes: \textit{incident location} (given by longitute and latitude), \textit{incident time} and \textit{incident type}. Since the attribute \textit{incident type} does not contain only crime incident types, but also other incident types, such as \textit{medical assist} or \textit{property found}, we preprocessed the dataset to obtain incidents of all 26 crime event types present in the dataset. The examples of event types are: \textit{aggravated assault}, \textit{arson}, \textit{auto theft} or \textit{drug charges}. 

Additionally, to better analyze the results that can be obtained with CSTS-Miner, we selected incidents of only the ten least frequent crime types in the complete dataset to create the \textit{reduced} dataset. These crime types are \textit{violation of liquor laws, operating under influence, manslaughter, homicide, harassment, gambling offense, embezzlement, crimes against children, bomb, arson}. 

In Table~\ref{Table:BostonCharacteristic}, we present the characteristic of the complete dataset, while in Table~\ref{Table:BostonCharacteristicReduced} the characteristic of the reduced dataset is shown. Since there are only 26 crime event types in the complete dataset, we decide to present the number of incidents of each type in the histogram shown in Figure~\ref{Fig:BostonHistogram}.

\begin{table}[h!t]
\caption{The characteristic of the \textit{complete} Boston Crime Incident Reports Dataset.}
\resizebox{\columnwidth}{!}{\begin{tabularx}{\linewidth}{lll}
	\toprule
    \textbf{Parameter} & \textbf{Value} \\
    \midrule
    No. of incident types (|\textbf{F}|) & 26 \\
    No. of incident instances (|\textbf{D}|) & 40~545 \\
    Avg. no. of incidents per type & 1559 \\
    Median no. of incidents per type & 527 \\
    Std. of no. of incidents per type & 2176 \\
    Min. no. of instances in a type & 1 \\
    Max. no. of instances in a type & 8575 \\
	\bottomrule
\end{tabularx}}
\label{Table:BostonCharacteristic}
\end{table}

\begin{table}[h!t]
\caption{The characteristic of the \textit{reduced} Boston Crime Incident Reports Dataset.}
\resizebox{\columnwidth}{!}{\begin{tabularx}{\linewidth}{lll}
	\toprule
    \textbf{Parameter} & \textbf{Value} \\
    \midrule
    No. of incident types (|\textbf{F}|) & 10 \\
    No. of incident instances (|\textbf{D}|) & 896 \\
    Avg. no. of incidents per type & 90 \\
    Median no. of incidents per type & 72 \\
    Std. of no. of incidents per type & 83 \\
    Min. no. of instances in a type & 1 \\
    Max. no. of instances in a type & 245 \\
	\bottomrule
\end{tabularx}}
\label{Table:BostonCharacteristicReduced}
\end{table}

\begin{figure*}[h!t]
	\centering 
	\includegraphics[width=0.95\textwidth]{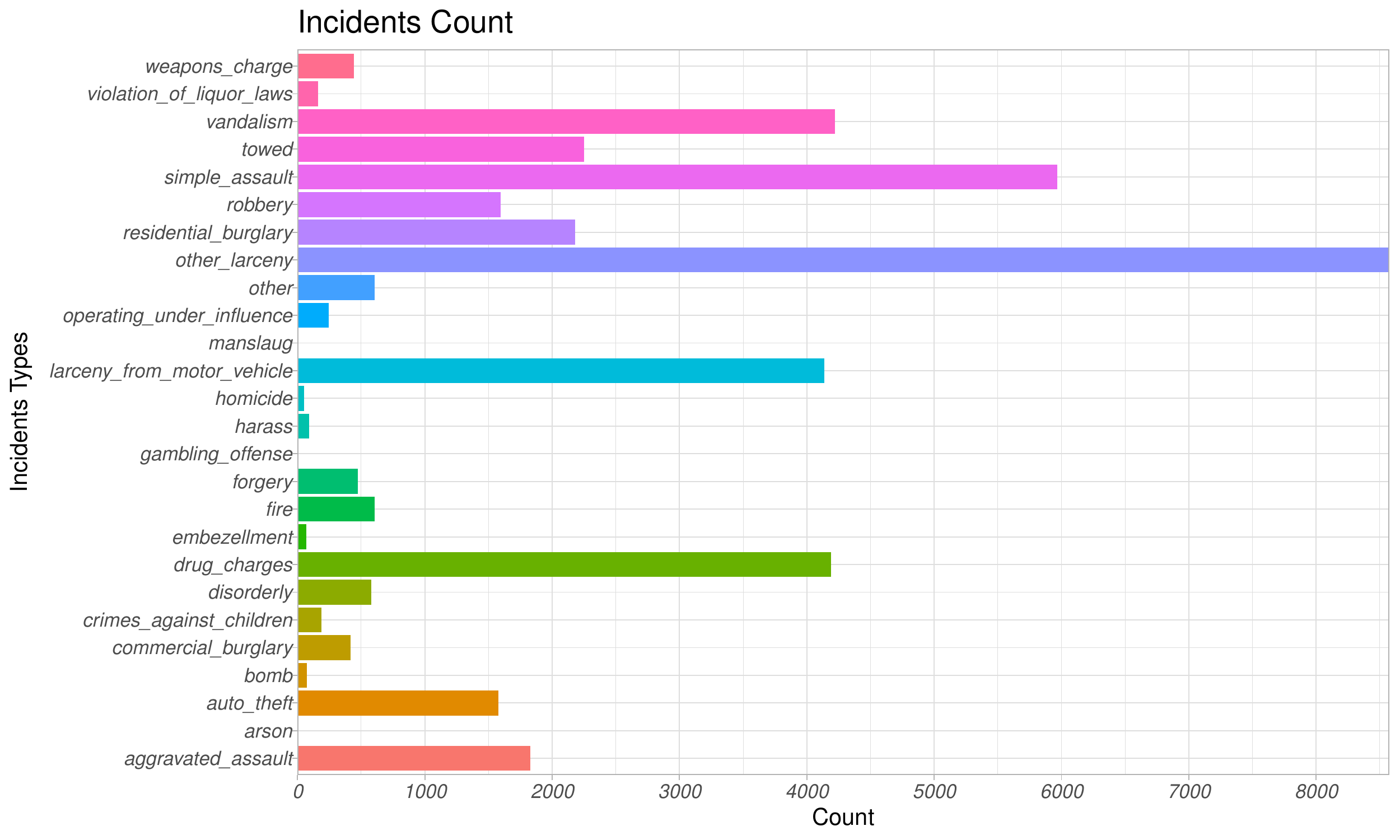}
	\caption{Incidents count per crime type in the \textit{complete} Boston Crime Incident Reports Dataset used for experiments.}
	\label{Fig:BostonHistogram}      
\end{figure*}

\subsection{Experimental Setup}

In our experiments, the spatial distance threshold $R$ of a neighborhood of an event instance $e$ is specified in meters. However, the locations of event instances in the obtained datasets are specified using the longitude and latitude coordinates. Thus, we apply the following procedure to transform the distance between two event instances $e_1, e_2 \in \mathbf{D}$ into meters. First, each coordinate (either longitude or latitude) of an instance $e$ is converted into radians according to Eq.~(\ref{Eq:DistanceRads}) in which, $e.lat$ refers to the latitude coordinate and $e.lon$ refers to the longitude coordinate of instance $e$, respectively. Next, the distance in meters between two instances is obtained according to Eq.~(\ref{Eq:Distance}). In Eq.~(\ref{Eq:Distance}), $\text{earthRadius}$ denotes the radius of Earth in kilometers. 

\begin{align}
    &  rad_{e.lat} = \frac{e.lat \cdot \pi}{180};~~rad_{e.lon} = \frac{e.lon \cdot \pi}{180}.  
 \label{Eq:DistanceRads}
\end{align}

\begin{align}
    & D_{slat} = \sin\big((rad_{e_2.lat} - rad_{e_1.lat}) / 2\big). \nonumber \\ 
    & D_{slon} = \sin\big((rad_{e_2.lon} - rad_{e_1.lon}) / 2\big). \nonumber \\ 
    & Dist_{(e_1, e_2)} = 2 \cdot \text{earthRadius} \cdot 1000~\cdot \nonumber \\
    & \arcsin\Big(\sqrt{D_{slat}^2 + D_{slon}^2 \cdot \cos(rad_{e_1.lat}) \cdot \cos(rad_{e_2.lat}})\Big). 
 \label{Eq:Distance}
\end{align}


The implementations of all algorithms selected for experiments are prepared in C++. We ran experiments using a computer equipped with Apple M1 processor and 16 GB of RAM memory. Our implementations of the algorithms (STS-Miner, STBFM, CSTPM, CST-SPMiner, CSTS-Miner) are available at the GitHub repository \footnote{https://github.com/piotrMaciag32/CSTS-Miner}. 

\subsection{Results of Experiments}

In the experiments, we compare our proposed CSTS-Miner with the four other algorithms: STS-Miner applying participation index measure \citep{ref1284:Huang2008}, STBFM \citep{ref1284:Maciag2019}, CSTPM \citep{ref1284:Mohan2012} which all discover PI-strong ST sequential patterns and  CST-SPMiner \citep{ref1284:Maciag2019-Kes} discovering PI-strong closed ST sequential patterns.

We aim to measure the number of discovered patterns by each of the algorithms and their computation times for each of the selected datasets. The main obtained results for each dataset are presented in Tables~\ref{Table:comptime1},~\ref{Table:comptime2} and~\ref{Table:comptime3}. The results presented in these tables were obtained for the following input parameters of the five compared algorithms:
\begin{itemize}
    \item For the Pittsburgh Police Incident Blotter Dataset: $R$ = 350 meters, $T$ = 11520 (8 days), $\theta$ = $\{0.33, 0.01, \dots, 0.25\}$, $\varepsilon$ = $\{0.025, 0.05, 0.075\}$.
    \item For the \textit{complete} Boston Crime Incidents Report Dataset: $R$ = 300 meters, $T$ = 5760 minutes (4 days), $\theta$ = $\{0.055, 0.05, \dots, 0.015\}$, $\varepsilon$ = $\{0.05, 0.1, 0.15\}$.
    \item For the \textit{reduced} Boston Crime Incidents Report Dataset: $R$ = 500 meters, $T$ = 43200 minutes (30 days), $\theta = \{0.01, 0.0095, \dots, 0.005\}$, $\varepsilon = \{0.05, 0.1, 0.15\}$.
\end{itemize}

Table~\ref{Table:comptime1} presents the results for the Pittsburgh Police Incident Blotted dataset. As it can be noted from the table, CSTS-Miner can discover much fewer patterns for all three selected values of its approximation margin $\varepsilon$ parameter than the other four selected algorithms. For example, for the participation index threshold $\theta = 0.26$, there are \num{143666} PI-strong ST sequential patterns and \num{100235} PI-strong closed ST sequential patterns, while the number of PI-strong CSTS patterns discovered for the approximation threshold value $\varepsilon = 0.1$ is only \num{58519}. As it can be noted from Table~\ref{Table:comptime1}, the increasing values of approximation margin $\varepsilon$ can result in a significant increase of computation times of CSTS-Miner. For example, for $\theta = 0.25$ STBFM and CST-SPMiner both executed in 207 seconds, while CSTS-Miner for $\varepsilon = 0.025$ parameter executed in 253 seconds and for $\varepsilon = 0.075$ parameter executed in 1744. Please note that for $\theta = 0.25$ computations for STS-Miner and CSTPM were too long to be obtained.


The slightly different results were presented in Table~\ref{Table:comptime2} for the \textit{complete} Boston Crime Incident Report dataset. In the case of this dataset, it was possible to obtain a similar reduction in the number of discovered patterns as in the case of the Pittsburgh Police Incident Blotted dataset. For example, for the participation index threshold $\theta$ equal to 0.015, the numbers of PI-strong ST sequential patterns and PI-strong closed ST sequential patterns are \num{2819490} and \num{2040303}, respectively. For the same value of $\theta$ parameter and $\varepsilon = 0.15$, CSTS-Miner is able to provide as many as \num{1171955} PI-strong CSTS patterns. Thus, the reduction in the number of patterns is 58\% when compared to the STBFM algorithm as well as 43\% when compared to the CST-SPMiner algorithm. 

Finally, in Table~\ref{Table:comptime3} we present the results obtained for the \textit{reduced} Boston Crime Incident Report dataset. The reduction in the number of discovered patterns is even more impressive in the case of this dataset. For the values of participation index threshold $\theta = 0.005$ and approximation margin $\varepsilon = 0.15$, CSTS-Miner provided \num{65899} PI-strong CSTS patterns. For the same $\theta$ value, the numbers of PI-strong ST sequential patterns and PI-strong closed ST sequential patterns discovered are \num{228285} and \num{76894} patterns, respectively. Thus, the reduction in the number of discovered patterns by CSTS-Miner when compared to STS-Miner, CSTPM, STBFM is 71\% and when to compared to CST-SPMiner is 24\%. However, please note that the computation time for the reduced dataset can be significantly higher in the case of the CSTS-Miner algorithm than in the cases of the STBFM and CST-SPMiner algorithms. Please also note that for the reduced Boston Crime Incident Report dataset, CSTPM executes much longer than the other algorithms.

In Figure~\ref{Fig:PercentSTBFM}, we provide the plots presenting the percent of the number of discovered PI-strong CSTS patterns to the number of discovered PI-strong ST sequential patterns. As it can be noted from the presented plots, the percent ranges between 15\% and 90\%. It was possible to obtain minimal values of the percent (around 15\%) for the reduced Boston Crime Incidents Report dataset when the spatial threshold $R = 600$ meters as well as temporal window $T = 28800$ minutes were applied and the $\theta$ threshold was equal to the values in the range $0.29 - 0.28$. In the case of both Pittsburgh Crime Incident Blotter and complete Boston Crime Incidents datasets, for the smaller values of the $\theta$ threshold ($< 0.3$) and the approximation margin $\varepsilon$ equal to $0.2$, the obtained percent is usually below 50\%. Interestingly, not always the smaller value of the $\theta$ threshold resulted in a more significant reduction of the number of discovered patterns. For example, for the Pittsburgh Crime Incident Blotter dataset and parameters $R = 350$ meters, $T = 11520$ minutes as well as approximation margin $\varepsilon$ equal to $0.2$ or $0.1$ the smallest percent was obtained for the $\theta = 0.28$.

\begin{landscape}
\begin{table*}[h!t]
	\scriptsize
	\caption{Computation time (in {sec.}) and the number of discovered patterns for the Pittsburgh Police Incident Blotter Dataset.}
	\begin{tabular}{ccccccccccccccc}
		\toprule
		\multicolumn{15}{c}{\textbf{ R = 350  meters,  T = 11520 (8 days)}}  \\ 
		\midrule
		&\multicolumn{6}{c}{\makecell[c]{\textbf{PI-strong}\\\textbf{ST seq. patterns}}}&\multicolumn{2}{c}{\makecell[c]{\textbf{PI-strong closed}\\\textbf{ST seq. patterns}}}&\multicolumn{6}{c}{\makecell[c]{\textbf{PI-strong CSTS patterns}}}\\
		\cmidrule(lr){2-7}\cmidrule(lr){8-9}\cmidrule(lr){10-15}
    	&\multicolumn{2}{c}{\textit{STS-Miner }}&\multicolumn{2}{c}{\textit{CSTPM }}&\multicolumn{2}{c}{\textit{STBFM }}&\multicolumn{2}{c}{\textit{CST-SPMiner }}&  \multicolumn{2}{c}{\makecell[c]{\textit{CSTS-Miner}\\ $(\varepsilon = 0.025)$}}&\multicolumn{2}{c}{\makecell[c]{\textit{CSTS-Miner}\\ $(\varepsilon = 0.05)$}}  &\multicolumn{2}{c}{\makecell[c]{\textit{CSTS-Miner}\\ $(\varepsilon = 0.075)$}} \\
		\cmidrule(lr){2-3}\cmidrule(lr){4-5}\cmidrule(lr){6-7}\cmidrule(lr){8-9}\cmidrule(lr){10-11}\cmidrule(lr){12-13}\cmidrule(lr){14-15}
		$ \theta $ &  time & \# patterns &  time & \# patterns & time & \# patterns &  time & \# patterns & time & \# patterns  & time & \# patterns & time & \# patterns \\
		\cmidrule{1-15}
0.33&	26& 796& 9& 796&        9& 796&	        9& 657&	        9& 594&	        9& 576&	        9& 556	\\
0.32&	31& 1002&9& 1002&       9& 1002&	    9& 819&	        9& 712&	        9& 681&	        9& 666	\\
0.31&	39& 1320&9& 1320&       9& 1320&	    9& 1065&	    9& 883&	        9& 827&	        9& 802	\\
0.3&	56& 1990&9& 1990&       9& 1990&	    9& 1573&	    9& 1161&	    9& 1050&	    9& 997	\\
0.29&	116& 4371&11& 4371&     10& 4371&	    10& 3323&	    10& 2235&	    10& 2018&	    10& 1932	\\
0.28&	269& 10434&17& 10434&   12& 10434&	    12& 7741&	    12& 4516&	    12& 4206&	    12& 4104	 \\
0.27&	898& 35166&100& 35166&  19& 35166&	    19& 25446&	    19& 14956&	    20& 14182&	    22& 13939	\\
0.26&	-&	-&1658& 143666&     47& 143666&	    47& 100235&	    51& 60977&	    67& 58907&	    99& 58519	\\
0.25&	-&	-&-&	-&          206& 1073917&	206& 683860&	253& 515128&	712& 509232&	1744& 508492 \\
	\bottomrule
\end{tabular}
\label{Table:comptime1}
	\centering
	\scriptsize
	\caption{Computation time (in {sec.}) and the number of discovered patterns for the \textit{complete} Boston Crime Incidents Dataset (- denotes that the computation times were too long to obtain the results).}
	\begin{tabular}{ccccccccccccccc}
		\toprule
		\multicolumn{15}{c}{\textbf{ R = 300  meters,  T = 5760 (4 days)}}  \\ 
		\midrule
		&\multicolumn{6}{c}{\makecell[c]{\textbf{PI-strong}\\\textbf{ST seq. patterns}}}&\multicolumn{2}{c}{\makecell[c]{\textbf{PI-strong closed}\\\textbf{ST seq. patterns}}}&\multicolumn{6}{c}{\makecell[c]{\textbf{PI-strong CSTS patterns}}}\\
		\cmidrule(lr){2-7}\cmidrule(lr){8-9}\cmidrule(lr){10-15}
    	&\multicolumn{2}{c}{\textit{STS-Miner }}&\multicolumn{2}{c}{\textit{CSTPM }}&\multicolumn{2}{c}{\textit{STBFM }}&\multicolumn{2}{c}{\textit{CST-SPMiner }}&  \multicolumn{2}{c}{\makecell[c]{\textit{CSTS-Miner}\\ $(\varepsilon = 0.1)$}}&\multicolumn{2}{c}{\makecell[c]{\textit{CSTS-Miner}\\ $(\varepsilon = 0.05)$}}  &\multicolumn{2}{c}{\makecell[c]{\textit{CSTS-Miner}\\ $(\varepsilon = 0.15)$}} \\
		\cmidrule(lr){2-3}\cmidrule(lr){4-5}\cmidrule(lr){6-7}\cmidrule(lr){8-9}\cmidrule(lr){10-11}\cmidrule(lr){12-13}\cmidrule(lr){14-15}
		$ \theta $ &  time & \# patterns &  time & \# patterns & time & \# patterns &  time & \# patterns & time & \# patterns  & time & \# patterns & time & \# patterns \\
		\cmidrule{1-15}
0.055&	69 &3982&      4 &3982  &          3& 3982&	        3& 3346&	       3& 2472 &	    3& 2232&	        3& 2170	\\
0.05&	92 &5386&      4 &5386&            3& 5386&	        4& 4484&	        3& 3265&	        3& 2977&	        3& 2913	\\
0.045&	129 &7550&     6 &7550&            4& 7550&	        4& 6248&	        4& 4410&	        4& 4060&	        4& 3991	\\
0.04&	199 &12014&    9 &12014&           5& 12014&	    5& 9899  &	            5& 6713&	        5& 6292&	        5& 6211	\\
0.035&	350 &22191&    23 &22191&          6& 22191&	    6& 17732&	            6& 11210&	    6& 10663&	    6& 10576	\\
0.03&	676 &43327&    77 &43327&          8& 43327&	    8& 33954&	            9& 21091&	    8& 20449&	    8& 20328	 \\
0.025&	- &-&           440 &102731&       14& 102731&	    14& 79701&	           14& 48775&	    14& 47964&	    14& 47797	\\
0.02 &	-&	-&          -& -&               33& 367400&	    32& 282156&	            33& 168511&	    33& 167361&	    34& 167139	\\
0.015&	-&	-&          -&	-&              145& 2819490&	142& 2040303&	        160& 1173939&	176& 1172294&	190& 1171955 \\
	\bottomrule
\end{tabular}
\label{Table:comptime2}
\end{table*}
\end{landscape}

\begin{landscape}
\begin{table*}[h!t]
	\centering
	\scriptsize
	\caption{Computation time (in {sec.}) and the number of discovered patterns for the \textit{reduced} Boston Crime Incidents Dataset.}
    \begin{tabular}{ccccccccccccccc}
		\toprule
		\multicolumn{15}{c}{\textbf{ R = 500  meters,  T = 43200 (30 days)}}  \\ 
		\midrule
		&\multicolumn{6}{c}{\makecell[c]{\textbf{PI-strong}\\\textbf{ST seq. patterns}}}&\multicolumn{2}{c}{\makecell[c]{\textbf{PI-strong closed}\\\textbf{ST seq. patterns}}}&\multicolumn{6}{c}{\makecell[c]{\textbf{PI-strong CSTS patterns}}}\\
		\cmidrule(lr){2-7}\cmidrule(lr){8-9}\cmidrule(lr){10-15}
    	&\multicolumn{2}{c}{\textit{STS-Miner }}&\multicolumn{2}{c}{\textit{CSTPM }}&\multicolumn{2}{c}{\textit{STBFM }}&\multicolumn{2}{c}{\textit{CST-SPMiner }}&  \multicolumn{2}{c}{\makecell[c]{\textit{CSTS-Miner}\\ $(\varepsilon = 0.05)$}}&\multicolumn{2}{c}{\makecell[c]{\textit{CSTS-Miner}\\ $(\varepsilon = 0.1)$}}  &\multicolumn{2}{c}{\makecell[c]{\textit{CSTS-Miner}\\ $(\varepsilon = 0.15)$}} \\
		\cmidrule(lr){2-3}\cmidrule(lr){4-5}\cmidrule(lr){6-7}\cmidrule(lr){8-9}\cmidrule(lr){10-11}\cmidrule(lr){12-13}\cmidrule(lr){14-15}
		$ \theta $ &  time & \# patterns &  time & \# patterns & time & \# patterns &  time & \# patterns & time & \# patterns  & time & \# patterns & time & \# patterns \\
		\cmidrule{1-15}
0.01&	6 &34102&      46& 34102 &         0& 34102&	    0& 9987&    9& 8629&	    13& 8580&	    17& 8575	\\
0.0095&	6& 34102&        46& 34102&          0& 34102&	    0& 9987&	9& 8629&	    13& 8580&	    17& 8575	\\
0.009&	6 &34102&       46& 34102&          0& 34102&	    0& 9987&	9& 8629&	    13& 8580&	    17& 8575	\\
0.0085&	6 &34102&       46& 34102&          0& 34102&	    0& 9987&	9& 8629&	    13& 8580&	    17& 8575	\\
0.008&	27 &141860&     1055& 141860&       0& 141860&	    0& 51219&	217& 44228&	    318& 44165&	    393& 44165	\\
0.0075&	27 &141860&     1048& 141860&       0& 141860&	    0& 51219&	215& 44228&	    324& 44165&	    393& 44165	 \\
0.007&	27 &141860&     1050& 141860&       0& 141860&	    0& 51219&	216& 44228&	    318& 44165&	    393& 44165	\\
0.0065 &27 &141860&     1099& 141860&       0& 141860&	    0& 51219&	216& 44228&	    317& 44165&	    393& 44165	\\
0.006&	30 &161238&     1394& 161238&       0& 161238&	    0& 56812&	277& 47407&	    405& 47347&	    487& 47342 \\
0.0055&	30 &161238&     1381& 161238&       0& 161238&	    0& 56812&	278& 47407&	    405& 47347&	    487& 47342 \\
0.005&	43 &228285&     2787& 228285&       1& 228285&	    1& 76894&	369& 65967&	    530& 65903&	    624& 65899 \\
	\bottomrule
\end{tabular}
\label{Table:comptime3}
\end{table*}
\end{landscape}

\begin{figure}[h!t]
	\centering 
	\includegraphics[width=0.6\linewidth]{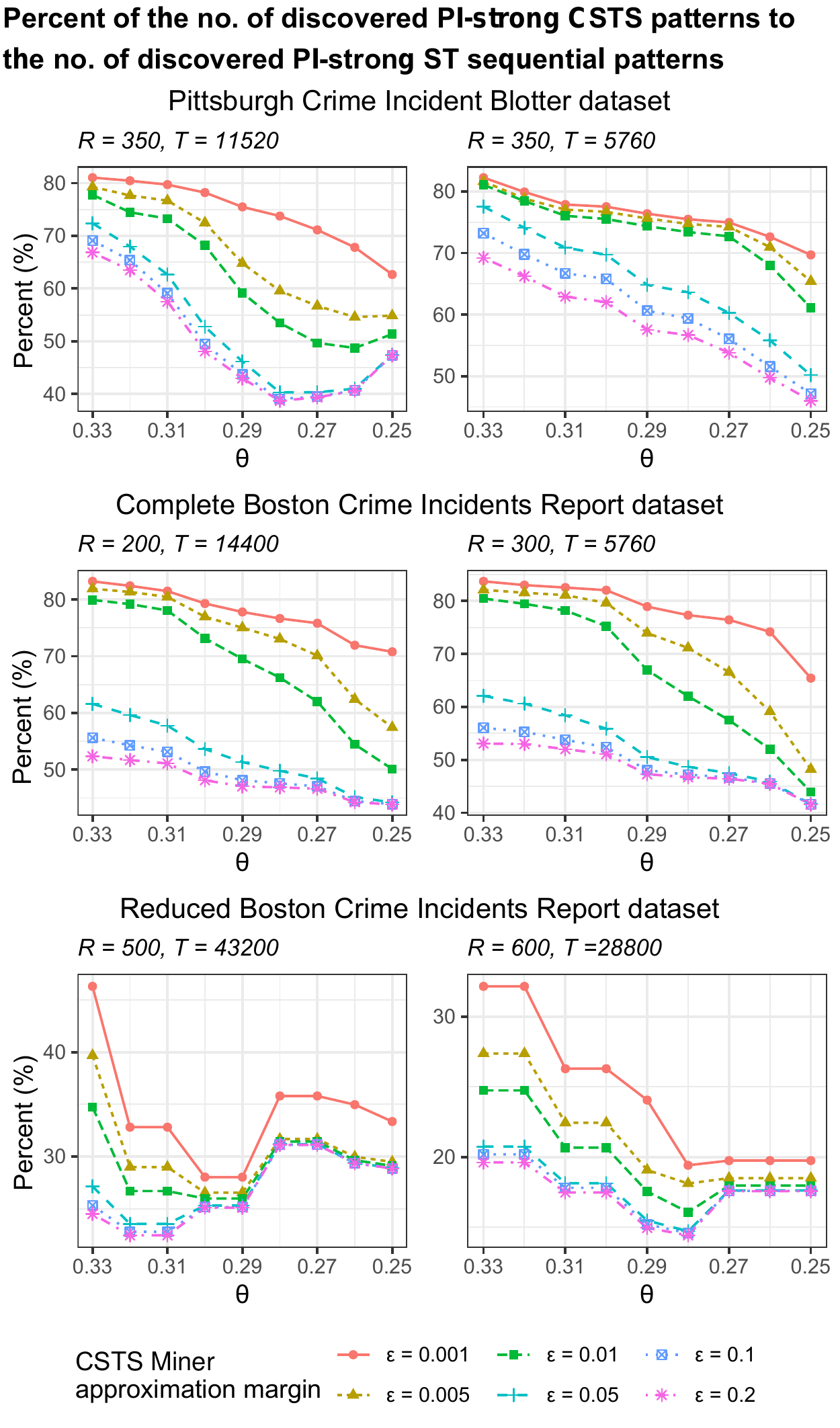}
	\caption{Percent of the number of PI-strong CSTS patterns discovered by the CSTS-Miner algorithm to the number of PI-strong ST sequential patterns discovered by the STS-Miner, STBFM and CSTPM algorithms.}
	\label{Fig:PercentSTBFM}      
\end{figure}

\begin{figure}[h!t]
	\centering 
	\includegraphics[width=0.6\linewidth]{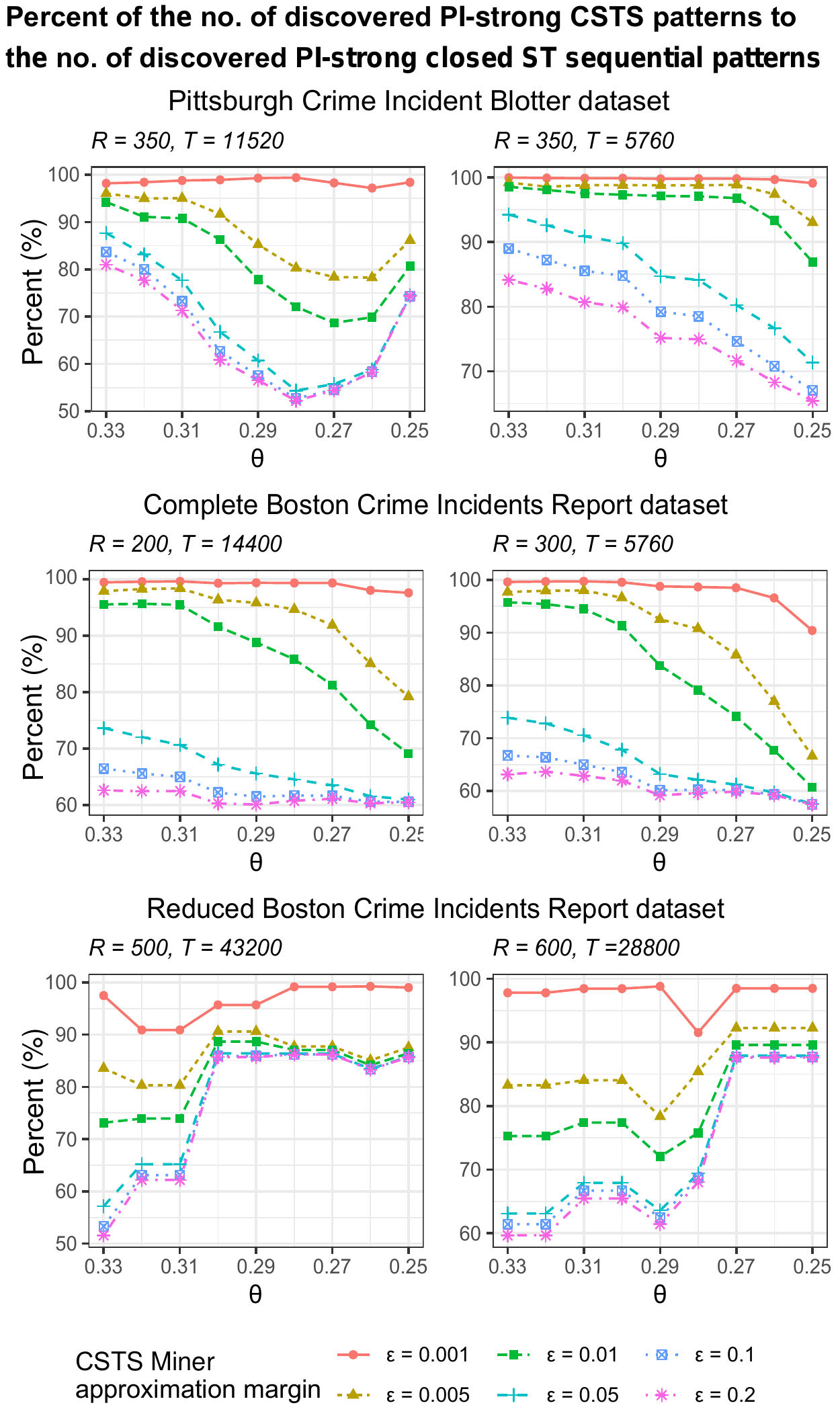}
	\caption{Percent of the number of PI-strong CSTS patterns patterns discovered by the CSTS-Miner algorithm to the number of PI-strong closed significant ST sequential patterns discovered by the CST-SPMiner algorithm.}
	\label{Fig:PercentCST}      
\end{figure}

\begin{figure*}[h!t]
	\centering 
	\includegraphics[width=1.0\linewidth]{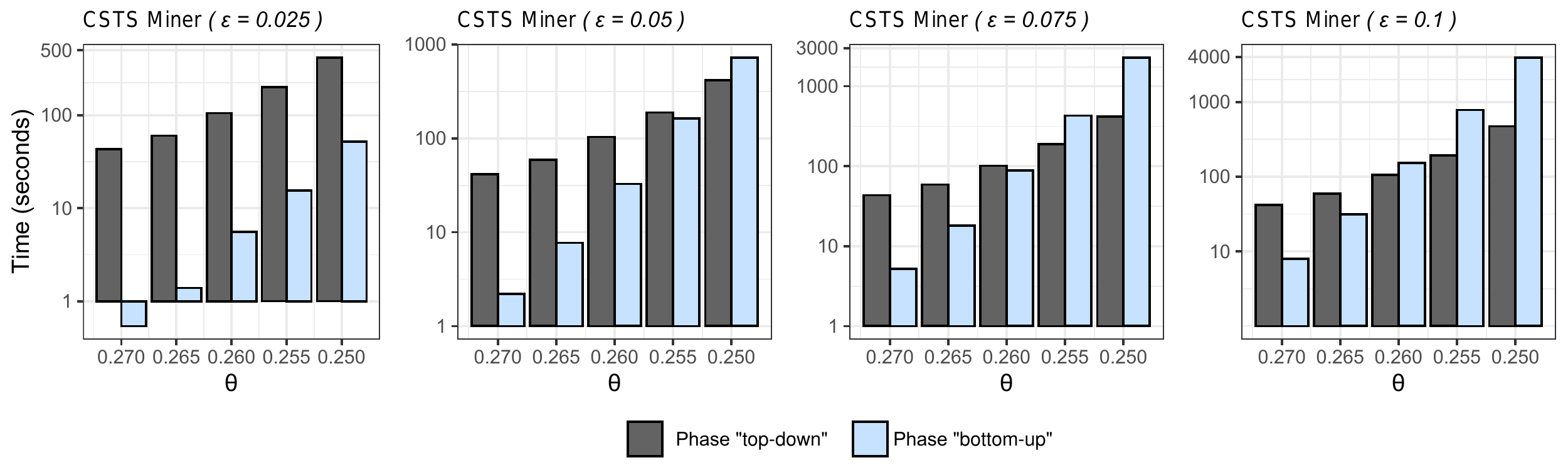}
	\caption{Computation time in seconds of the "top-down" phase (steps 1 - 17) and the "bottom-up" phase (steps 18 - 24) of Algorithm~
	\ref{Alg:CSTS-Miner} for the Pittsburgh Police Incident Blotter Dataset. Please note the logarithmic scale for the computation time.}
	\label{Fig:PhaseTimes}      
\end{figure*}



In Figure~\ref{Fig:PercentCST}, we provide the plots presenting the percent of the number of PI-strong CSTS patterns to the number of PI-strong closed ST sequential patterns. As follows from Lemma~\ref{Lem:CSTSasCLosed}, for $\varepsilon = 0$ the set of PI-strong CSTS patterns is equal to the set of PI-strong closed ST sequential patterns. However, for the greater values of the $\varepsilon$ parameter (such as, for example $0.1$ or $0.2$), CSTS-Miner is able to provide as few as 50\% of the number of PI-strong closed ST sequential patterns discovered by CST-SPMiner. Importantly, even for the smaller values of the $\varepsilon$ parameter (such as, for example, $0.01$), CSTS-Miner is able to provide as few as 70\% of the number of patterns provided by CST-SPMiner (as it is shown, for example, in the plots for the complete Boston Crime Incidents Report dataset presented in Figure~\ref{Fig:PercentCST}).

To summarize, the plots presented in Figures~\ref{Fig:PercentSTBFM} and~\ref{Fig:PercentCST} show that the CSTS-Miner algorithm can provide significantly fewer PI-strong CSTS patterns than PI-strong ST sequential patterns and PI-strong closed ST sequential patterns, even when the small values of the $\varepsilon$ parameter are applied.

In our next experiment, we aimed to compare the computation times of steps 1 - 18 (phase "top-down") and steps 19 - 25 (phase "bottom-up") of Algorithm~\ref{Alg:CSTS-Miner}. As previously noted, the "top-down" phase is responsible for generating all PI-strong ST sequential patterns, while the "bottom-up" phase calculates those of them which are also PI-strong CSTS patterns.

In Figure~\ref{Fig:PhaseTimes}, we present the comparison of computation times (in logarithmic scale) of both phases obtained for the Pittsburgh Police Incident Blotter dataset. Please note that for the smaller values of the approximation margin $\varepsilon$ (such as $\varepsilon = 0.025$), the computation times for the "top-down" phase are more significant than for the "bottom-up" phase. However, with the increasing values of $\varepsilon$ and for smaller values of $\theta$ threshold, the computation time of the "bottom-up" phase can increase significantly. For example, for the parameters $\varepsilon$ equal $0.1$ and $\theta$ equal $0.25$, the computation time of the "bottom-up" phase of Algorithm~\ref{Alg:CSTS-Miner} can be up to four times longer than the computation time of the "top-down" phase for the same values of these parameters.

\subsection{Selection of Representative Patterns}

In this subsection, we provide some interesting examples of discovered sequential patterns of different types of crimes. To do this end, we ran CSTS-Miner using the Pittsburgh Police Incident Blotter Dataset with the following parameters: $R = 300,$ $T = 11520$ minutes (8 days), $\varepsilon = 0.05,$ $\theta = 0.3$.

The examples of interesting resultant patterns include:
\begin{enumerate}[leftmargin=*]
    \item \small{\textit{public\_drunkenness} $\rightarrow$ \textit{robbery/ bank/ knife} (PI = 0.5).}
    \item \small{\textit{public\_drunkenness} $\rightarrow$ \textit{robbery/ bank/ strongarm} (PI = 0.0.44).}
    \item \small{\textit{simple\_assault/ injury $\rightarrow$ public\_drunkenness $\rightarrow$ public\_drunkenness $\rightarrow$ all\_other\_offenses (expt\_traff) $\rightarrow$ fail\_ disord\_per\_to\_disperse} (PI = 0.30).}
    \item \small{\textit{simple\_assault/ injury $\rightarrow$ public\_drunkenness $\rightarrow$ robbery/ bank/ \_strongarm} (PI = 0.33).}
    \item \small{\textit{robbery/ highway/ gun $\rightarrow$ sale/ use\_of\_air\_rifles} (PI = 0.5).}
\end{enumerate}

Patterns 1 and 2 could provide essential information about types of banks robberies. Pattern 1 states that half of the bank robberies using a knife were conducted within 300 meters and 8 days from the reported public drunkenness incidents. On the other hand, pattern 2 communicates that 44\% of all bank robberies using weapon (strongarm) occurred within 300 meters and 8 days from the reported public drunkenness incidents. The other interesting example is pattern 4, which states that half of the usage of air rifles occurred within 300 meters and 8 days from the highway robbery.


\section{Conclusion}
\label{sec:Conclusions}

In this article, we offered a new type of ST sequential patterns called $\varepsilon$-constricted ST sequential patterns (CSTS patterns) and we thoroughly analyzed their theoretical properties. Specifically, we showed that the set of CSTS patterns is a subset of the set of closed ST sequential patterns and that each CSTS pattern is also a closed ST sequential pattern. Moreover, we showed that given the set of PI-strong CSTS patterns one can obtain the set of all PI-strong ST sequential patterns and approximate participation index of each of them with the approximation margin $\pm~\varepsilon$. We also offered a new algorithm called CSTS-Miner that discovers all PI-strong CSTS patterns. To this end, CSTS-Miner adapts the MAX-Tree structure for the more efficient candidate patterns generation. The proposed MAX-Tree is generated in the two main phases of CSTS-Miner: the first one called "top-down" in which all PI-strong ST sequential patterns are generated using the breadth-first strategy, and the second one called "bottom-up" which calculates PI-strong ST sequential patterns being CSTS patterns. We analyzed properties and computation times of CSTS-Miner.

The experiments with the CSTS-Miner algorithm were conducted using two crime-related datasets for the cities of Boston and Pittsburgh: the Pittsburgh Police Incident Blotter Dataset and the Boston Crime Incident Reports Dataset. Each of the selected datasets consists of various types of crime and numerous event instances. 
Moreover, to better verify the capabilities of the proposed algorithm, we extracted a reduced dataset from the complete Boston Crime Indecent Reports dataset. The resultant reduced dataset contains 10 least frequent crime event types and 896 event instances. 

In the experimental evaluation, we compared the results obtained with the proposed CSTS-Miner algorithm with four other state-of-the-art algorithms: STS-Miner \citep{ref1284:Huang2008}, CSTPM \citep{ref1284:Mohan2012}, STBFM \citep{ref1284:Maciag2019}  and CST-SPMiner \citep{ref1284:Maciag2019-Kes}. The STS-Miner, CSTPM and STBFM algorithms discover PI-strong ST sequential patterns, whereas CST-SPMiner discovers PI-strong closed ST sequential patterns. Each of the selected algorithms, as well as the proposed algorithm, use the participation index to measure the significance of the discovered patterns. As we presented in the experiments, in the case of the Pittsburgh Police Incident Blotter Dataset and in the cases of the complete and reduced Boston Crime Incident Reports Dataset, CSTS-Miner can return much fewer patterns than the other selected algorithms. In particular, in the case of the Pittsburgh Police Incident Blotter Dataset, CSTS-Miner provides up to 60\% fewer patterns when compared to  STBFM and up to 50\% fewer patterns when compared to CST-SPMiner. Similarly, for the complete Boston Crime Incident Reports Dataset, CSTS-Miner provides up to 60\% fewer patterns when compared to STBFM and up to 40\% fewer patterns when compared to CST-SPMiner. For the reduced Boston Crime Incident Reports Dataset, CSTS-Miner provides up to 85\% fewer patterns than STBFM and up to 50\% fewer patterns than CST-SPMiner.


\section*{Acknowledgements}

This work was supported by the RENOIR (Reverse Engineering of Social Information Processing) program and by the Institute of Computer Science of Warsaw University of Technology. 


\bibliographystyle{model5-names}
\biboptions{authoryear}
\bibliography{main}
\vfill

\end{document}